\documentclass{article}

\usepackage{graphicx}
\usepackage[utf8]{inputenc} % utf8 coding
\usepackage{amssymb,amsmath,amsthm} % math related stuff
\usepackage{graphicx} % including graphics
\usepackage{array} % array package
\usepackage{eurosym} % the euro-symbol
\usepackage{xcolor} % colored text
\usepackage{url}
\usepackage[ruled,linesnumbered,vlined]{algorithm2e} % Pseudocode
\usepackage{caption} % support for captions in sub figures
\usepackage{subcaption} % sub figures
\usepackage{authblk} % title formatting
\usepackage{hyperref} % pdf links for citations and in-document references

\usepackage{tabularx}
\newcolumntype{L}[1]{>{\raggedright\arraybackslash}p{#1}} % left aligned with fixed column-width
\newcolumntype{C}[1]{>{\centering\arraybackslash}p{#1}} % centered with fixed column-width
\newcolumntype{R}[1]{>{\raggedleft\arraybackslash}p{#1}} % right aligned with fixed column-width
\usepackage{colortbl}
\usepackage{multirow}
\usepackage{rotating}

\usepackage{pifont}
\newcommand{\leadingzero}[1]{\ifnum #1<10 0\the#1\else\the#1\fi} 

\newtheorem{theorem}{Theorem}
\definecolor{darkgrey}{RGB}{28,28,28} % HTML: #1c1c1c
\definecolor{mediumgrey}{RGB}{71,71,71} % HTML: #474747
\definecolor{lightgrey}{RGB}{115,115,115} % HTML: #737373
\definecolor{darkblue}{RGB}{79,101,140} % HTML: #4f658c
\definecolor{mediumblue}{RGB}{112,144,200} % HTML: #7090c8
\definecolor{lightblue}{RGB}{140,180,250} % HTML: #8cb4fa
\definecolor{darkgreen}{RGB}{41,115,46} % HTML: #29732e
\definecolor{mediumgreen}{RGB}{65,180,73} % HTML: #41b449
\definecolor{lightgreen}{RGB}{90,250,101} % HTML: #5afa65
\definecolor{mediumyellow}{RGB}{247,203,56} % HTML: #f7cb38
\definecolor{mediumorange}{RGB}{255,138,60} % HTML: #ff8a3c
\definecolor{mediumred}{RGB}{219,73,55} % HTML: #db4937
\definecolor{mediumviolet}{RGB}{146,90,199} % HTML: #925ac7
\definecolor{mediumturquoise}{RGB}{87,189,227} % HTML: #57bde3

\DeclareMathOperator*{\arctantwo}{\text{atan}2}

\newcommand{\mContinuousGeZero}{{\mathbb{R}^{+}_{0}}}
\newcommand{\mContinuousPositive}{{\mathbb{R}^{+}}}
\newcommand{\mContinuous}{{\mathbb{R}}}

\newcommand{\mIndexTime}{t}
\newcommand{\mSetTime}{\mathbb{T}}
\newcommand{\mIndexBot}{r}
\newcommand{\mSetBots}{\mathcal{R}}
\newcommand{\mIndexBucket}{b}
\newcommand{\mSetBuckets}{\mathcal{B}}
\newcommand{\mIndexStation}{m}
\newcommand{\mSetStations}{\mathcal{M}}
\newcommand{\mSetIStations}{\mSetStations^{I}}
\newcommand{\mSetOStations}{\mSetStations^{O}}
\newcommand{\mIndexElevator}{l}
\newcommand{\mSetElevators}{\mathcal{L}}
\newcommand{\mIndexTier}{h}
\newcommand{\mSetTiers}{\mathcal{H}}
\newcommand{\mIndexWaypoint}{w}
\newcommand{\mSetWaypoints}{\mathcal{W}}
\newcommand{\mSetElevatorWaypoints}[1]{\mSetWaypoints_{#1}}
\newcommand{\mIndexEdge}{e}
\newcommand{\mSetEdges}{\mathcal{E}}
\newcommand{\mSetEdgesReverse}{\mSetEdges^{-1}}
\newcommand{\mSetGraph}{\mathcal{G}}
\newcommand{\mSetGraphReverse}{\mathcal{G}^{-1}}

\newcommand{\mSetBucketParkingSpots}{\mSetWaypoints^{P}}

\newcommand{\mSubTaskMove}{{move}}
\newcommand{\mSubTaskPickup}{{pickup}}
\newcommand{\mSubTaskSetdown}{{setdown}}
\newcommand{\mSubTaskPut}{{put}}
\newcommand{\mSubTaskGet}{{get}}

\newcommand{\mSetSearchSpace}{\mathcal{S}}
\newcommand{\mIndexState}{n}
\newcommand{\mSetBlockedNodes}{\mathcal{J}}

\newcommand{\mVarPosTier}[2]{p^H_{#1 #2}}
\newcommand{\mVarPosX}[2]{p^X_{#1 #2}}
\newcommand{\mVarPosY}[2]{p^Y_{#1 #2}}
\newcommand{\mVarPos}[2]{\left( \mVarPosTier{#1}{#2}, \mVarPosX{#1}{#2}, \mVarPosY{#1}{#2} \right)}
\newcommand{\mVarPosShort}[2]{p_{#1 #2}}
\newcommand{\mVarPosTierFixed}[1]{p^H_{#1}}
\newcommand{\mVarPosXFixed}[1]{p^X_{#1}}
\newcommand{\mVarPosYFixed}[1]{p^Y_{#1}}
\newcommand{\mVarPosFixed}[1]{\left( \mVarPosTierFixed{#1}, \mVarPosXFixed{#1}, \mVarPosYFixed{#1} \right)}
\newcommand{\mVarPosFixedShort}[1]{p_{#1}}
\newcommand{\mVarSpeed}[2]{v_{#1 #2}}

\newcommand{\mVarOrientation}[2]{o_{#1 #2}}
\newcommand{\mVarAngularVelocity}[2]{\omega_{#1 #2}}
\newcommand{\mVarPriority}[1]{p^O_{#1}}
\newcommand{\mVarResTab}{r^T}
\newcommand{\mVarTimeStampeOD}[2]{t_{#1 #2}}

\newcommand{\mValAngle}{\varphi}
\newcommand{\mValPath}[2]{\pi_{#1 #2}}
\newcommand{\mValPathIndexed}{\mValPath{\mIndexBot}{\mIndexTime}}

\newcommand{\mParWaypointInitial}[1]{V^I_{#1}}

\newcommand{\mParTimePickupSetdown}{T^{B}}
\newcommand{\mParTimeOStationHandleUnit}[1]{T^{O}_{#1}}
\newcommand{\mParTimeOStationHandleUnitIndexed}{\mParTimeOStationHandleUnit{\mIndexStation}}
\newcommand{\mParTimeIStationHandleUnit}[1]{T^{I}_{#1}}
\newcommand{\mParTimeIStationHandleUnitIndexed}{\mParTimeIStationHandleUnit{\mIndexStation}}
\newcommand{\mParTimeUseElevator}[3]{T^{L}_{#1 #2 #3}}

\newcommand{\mParTimeWaiting}{T^{W}}
\newcommand{\mParTimeRunning}{T^{R}}
\newcommand{\mParAcceleration}[1]{\overrightarrow{A}_{#1}}

\newcommand{\mParDeceleration}[1]{\overleftarrow{A}_{#1}}

\newcommand{\mParMaxVelocity}[1]{\overline{V}_{#1}}

\newcommand{\mParMaxAngularVelocity}[1]{\Omega_{#1}}

\newcommand{\mParLength}[1]{L^X_{#1}}
\newcommand{\mParWidth}[1]{L^Y_{#1}}
\newcommand{\mParRadius}[1]{L^R_{#1}}

\newcommand{\mParIteration}{I}
\newcommand{\mParPathPlanningTimeout}{T^{T}}

\newcommand{\mGetBucketOwner}[2]{f^O \left(#1,#2\right)}
\newcommand{\mGetBucketOwnerDefinition}{\mGetBucketOwner{\mIndexBucket}{\mIndexTime} : \mSetBuckets \times \mSetTime \rightarrow \mSetBots \cup \mSetBucketParkingSpots}
\newcommand{\mGetBotIsCarrying}[2]{f^C \left(#1,#2\right)}

\newcommand{\mGetFixedReservation}[1]{f^R \left(#1\right)}
\newcommand{\mGetFixedAndFinalReservation}[1]{f^{FR} \left(#1\right)}
\newcommand{\mRemoveFinalReservation}[2]{f^D \left(#1,#2\right)}
\newcommand{\mGetPosition}[1]{f^P \left(#1\right)}
\newcommand{\mSortAgentwithPriority}[2]{f^S \left(#1,#2\right)}
\newcommand{\mSortAgent}[1]{f^S \left(#1\right)}
\newcommand{\mGetDestinationIsChanged}[1]{f^{DC} \left(#1\right)}
\newcommand{\mGetHop}[2]{f^{HOP} \left(#1,#2\right)}
\newcommand{\mGetWaitsteps}[2]{f^W \left(#1,#2\right)}
\newcommand{\mAddReservation}[2]{f^{AR} \left(#1,#2\right)}
\newcommand{\mReorganizeReservationTable}[1]{f^{RR} \left(#1\right)}
\newcommand{\mAddFinalReservation}[1]{f^{AFR} \left(#1\right)}
\newcommand{\mAddWaitAction}[1]{f^{AW} \left(#1\right)}
\newcommand{\mUseEvasionStrategy}[1]{f^{E} \left(#1\right)}
\newcommand{\mGetRelationIsCircle}[2]{f^{RC} \left(#1,#2\right)}
\newcommand{\mGetLastAction}[1]{f^{GA} \left(#1\right)}
\newcommand{\mDetectCollision}[2]{f^{DC} \left(#1,#2\right)}
\newcommand{\mGetCollision}[2]{f^{GC} \left(#1,#2\right)}
\newcommand{\mGetCost}[2]{f^{C} \left(#1,#2\right)}

\newcommand{\mGetDistance}[3]{d^E \left(#1,#2,#3\right)}
\newcommand{\mGetDistanceTimeIndependent}[2]{d^E \left( #1,#2 \right)}

\newcommand{\mGetTimeRotation}[3]{t^R \left( #1, #2, #3 \right)}
\newcommand{\mGetTimeRotationDefinition}{\mGetTimeRotation{\mIndexBot}{\mValAngle}{\mParMaxAngularVelocity{\mIndexBot}} := \mValAngle \mParMaxAngularVelocity{\mIndexBot}^{-1}}
\newcommand{\mGetTimeRotationDefinitionAppendix}{\mIndexBot \in \mSetBots, \mValAngle, \mParMaxAngularVelocity{\mIndexBot} \in \mContinuousGeZero}
\newcommand{\mGetTimeDrive}[5]{t^D \left( #1, #2, #3, #4, #5 \right)}
\newcommand{\mGetTimeDriveDefinition}[1]{\mGetTimeDrive{\mIndexBot}{#1}{\mParAcceleration{\mIndexBot}}{\mParMaxVelocity{\mIndexBot}}{\mParDeceleration{\mIndexBot}} := 
\begin{cases}
\frac{\mParMaxVelocity{\mIndexBot}}{\mParAcceleration{\mIndexBot}} + \frac{#1 - \frac{\mParAcceleration{\mIndexBot}}{2} \left( \frac{\mParMaxVelocity{\mIndexBot}}{\mParAcceleration{\mIndexBot}} \right)^2 - \frac{\mParDeceleration{\mIndexBot}}{2} \left( \frac{\mParMaxVelocity{\mIndexBot}}{\mParDeceleration{\mIndexBot}} \right)^2}{\mParMaxVelocity{\mIndexBot}} + \frac{\mParMaxVelocity{\mIndexBot}}{\mParDeceleration{\mIndexBot}} & \text{if }#1 \ge \tilde{#1}\\[10pt]
\sqrt{\frac{#1}{\frac{\mParAcceleration{\mIndexBot}}{2} + \frac{\mParAcceleration{\mIndexBot}^2}{2\mParDeceleration{\mIndexBot}}}} + \sqrt{\frac{#1}{\frac{\mParDeceleration{\mIndexBot}}{2} + \frac{\mParDeceleration{\mIndexBot}^2}{2\mParAcceleration{\mIndexBot}}}} & \text{if }#1 < \tilde{#1}
\end{cases}}

\newcommand{\mGetTimeDriveDefinitionMinFullAccelerationDistance}[1]{\tilde{#1} := \frac{\mParAcceleration{\mIndexBot}}{2} \left( \frac{\mParMaxVelocity{\mIndexBot}}{\mParAcceleration{\mIndexBot}} \right)^2 + \frac{\mParDeceleration{\mIndexBot}}{2} \left( \frac{\mParMaxVelocity{\mIndexBot}}{\mParDeceleration{\mIndexBot}} \right)^2}
\newcommand{\mGetAStarCostsEstimatedByWHCA}[1]{h^{\text{\textit{WHCA}}^*} \left(#1\right)}
\newcommand{\mGetAStarCostsEstimatedByRRA}[1]{h^{RRA*} \left(#1\right)}
\newcommand{\mGetAStarCostsEstimatedTwoD}[1]{h^{2D} \left(#1\right)}

\newcommand{\mDecisionTaskAllocation}[2]{\alpha^{TA} (#1,#2)}

\graphicspath{{media/}}

\begin{document}
	
	\title{Multi-Agent Path Finding \\
		with Kinematic Constraints for \\Robotic Mobile Fulfillment Systems}
	
	\author[1]{Marius Merschformann\footnote{marius.merschformann@uni-paderborn.de}}
	\author[2]{Lin Xie\footnote{lin.xie@leuphana.de}}
	\author[1]{Daniel Erdmann}
	\affil[1]{University of Paderborn, Paderborn, Germany}
	\affil[2]{Leuphana University of Lüneburg, Lüneburg, Germany}
	
	\maketitle
	
	\begin{abstract}
	This paper presents a collection of path planning algorithms for real-time movement of multiple robots across a Robotic Mobile Fulfillment System (RMFS). In such a system, robots are assigned to move storage units to pickers at working stations instead of requiring pickers to go to the storage area. Path planning algorithms aim to find paths for the robots to fulfill the requests without collisions or deadlocks. The robots are fully centralized controlled. The traditional path planning algorithms do not consider kinematic constraints of robots, such as maximum velocity limits, maximum acceleration and deceleration, and turning time. This work aims at developing new multi-agent path planning algorithms by considering kinematic constraints. Those algorithms are based on some exisiting path planning algorithms in literature, including WHCA*, FAR, BCP, OD\&ID and CBS. Moreover, those algorithms are integrated within a simulation tool to guide the robots from their starting points to their destinations during the storage and retrieval processes. Ten different layouts with a variety of numbers of robots, floors, pods, stations and the sizes of storage areas were considered in the simulation study. Performance metrics of throughput, path length and search time were monitored. Simulation results demonstrate the best algorithm based on each performance metric.
	%\keywords{Multi-Agent pathfinding \and Robotic Mobile Fulfillment Systems \and Warehouse \and Simulation}
	\end{abstract}
	
	\section{Introduction}\label{sec:pp_intro}
	Due to the increasingly fast-paced economy, an efficient distribution center plays a crucial role in the supply chain. From the logistics perspective the main task is to turn homogeneous pallets into ready-to-ship packages that will be sent to the customer. Traditionally, as some customers' orders are received, pickers in different zones of a warehouse are sent to fetch the products, which are parts of several different customers' orders. After that, the products should be sorted and scanned. Once all parts of an order are complete, they are sent to packing workers to finish packaging. An extensive overview of manual order picking systems can be found in \cite{LeAnh.2006}. As shown in \cite{Tompkins.2010}, 50\% of pickers' time in these systems is spent on traveling around the warehouse. To ensure that the orders are shipped as fast as possible, automated storage and retrieval systems were introduced. An extensive literature review is provided by \cite{Roodbergen.2009}. While using these systems offers a high potential throughput they also face certain drawbacks such as high costs, long design cycles, inflexibility and lack of expandability as pointed out by \cite{Wurman.2008}. In order to improve or eliminate those disadvantages, automated Robotic Mobile Fulfillment Systems (RMFS), such as the Kiva System (\cite{Enright.2011}, nowadays Amazon Robotics), have been introduced as an alternative order picking system in recent years. Robots are sent to carry storage units, so-called ``pods'', from the inventory and bring them to human operators, who work at picking stations. At the stations, the items are packed according to the customers' orders. \cite{Wurman.2008} indicate that this system increases the productivity two to three times, compared with the classic manual order picking system. Moreover, the search and travel tasks for the pickers are eliminated. 

The applications in a similar RMFS have been paid more attention recently (see \cite{Merschformann.2017b} for a better overview). An efficient path-planning algorithm is important, since it aims at finding paths for robots to fulfill the requests without collisions or deadlocks, which is considered a major aspect of automation for storage and retrieval in an RMFS. The existing publications formulate path planing in such system as the classic Multi-Agent Pathfinding problem (MAPF) and this problem is solved with a bounded sub-optimal solver (see \cite{Cohen.2015} and \cite{Cohen.2017}). MAPF is a challenging problem with many applications in robotics, disaster rescue and video games (see \cite{Wang.2011}). However, MAPF solvers from AI typically do not work with real-world mobile robots, since they do not work for agents considering kinematic constraints, such as maximum velocity limits, maximum acceleration and deceleration, and turning times, in continuous environments. The authors of \cite{Hoenig.} suggested a postprocessing to create a schedule that can be executed by robots, based on the output of a MAPF solver. Differ to that, we introduce a novel mathematical formulation of MAPF, called Multi-Agent Pathfinding Problem for mobile Warehousing Robots (MAPFWR for short), that takes some of the kinematic constraints of real robots as well as the structure of the simulated warehouse environment into account during the path-planning phase. The existing MAPF-algorithms, including WHCA*,FAR, BCP, OD\&ID and CBS, are modifed for this purpose. There are two existing open-source frameworks to simulate an RMFS-System, the first framework, ``Alphabet Soup'' was published by \cite{Hazard.2006}, and the framework ``RawSim-O'' (\cite{Merschformann.2017b}). This latter one is based on the former one with extending cases of multiple floors and they consider different robot movement emulations. In this work, we test different MAPFWR-algorithms in the framwork ``RawSim-O'', since it has realistic robot movement emulation by considering the robot's turning time and adjusted acceleration or deceleration formulas. 
%The study reported in the paper aims at developing an integrated suite of path-planning algorithms to manage storage and retrieval processes in real time for RMFS.  Moreover, the novel comparison of state-of-the-art path planning algorithms is shown to fulfill different requirements of such a system, namely throughput, path length and search time.
%Additionally, a new mathematical model is proposed in this work. In such a system, multiple decision problems have to be solved and coordinated effectively to ensure overall performance, such as path planning and task allocation. However, for this work we fix all remaining decision mechanisms to certain default methods (see Section \ref{sec:pp_Simulation}) and only alter the path planning method used. 

The first part of the paper gives a mathematical description of the MAPFWR problem and a literature review of the related problem MAPF. Next, we investigate the properties of the search space of MAPFWR. After that, we describe the path-planning algorithms that we implement. Finally, we present experimental results obtained by the simulation framework ``RawSim-O''.

	%%%%%%%%%%%%%%%%%%%%%%%%%%%%%%%%%%%%%%%%%%%%%%%%%%%%%%%%%
	%%%%%%%%%%%%%%%%%%%%%%%%%%%%%%%%%%%%%%%%%%%%%%%%%%%%%%%%%
	%%%%%%%%%%%%%%%%%%%%%%%%%%%%%%%%%%%%%%%%%%%%%%%%%%%%%%%%%
	
	\section{Background} \label{sec:pp_Problem}
	This section first describes RMFS and the inherent decision problems in more detail. Next, the MAPFWR problem and its related problem, the MAPF problem, is defined. At last, a literature review of MAPF is presented.
\subsection{Robotic Mobile Fulfillment System}\label{subsec:pp_RMFS}

According to \cite{Gu.2007} the ``(...) basic requirements in warehouse operations are to receive Stock Keeping Units (SKUs) from suppliers, store the SKUs, receive orders from customers, retrieve SKUs and assemble them for shipment, and ship the completed orders to customers.'' Using this definition an RMFS provides the temporal storage of SKUs, and their retrieval to fulfill incoming customer orders. The approach for this requirement is to use pods (shelf-like storage units) to store the inventory and bring these to replenishment and pick stations as they are required for insert and extract transactions of items. The basic layout of one floor in a RMFS is illustrated in Fig. \ref{fig:pp_exampleprobleminstance}. The pods are located at the storage area in the middle of the layout. The robot carries a pod by following waypoints to the replenishment station on the left-hand side, where new items are inserted into the pod. After that, the pod is carried back to the storage area. Similar operations are done at the order picking station on the right-hand side, where items are picked to fulfill orders. The robots are guided by waypoints, which are typically connected by a grid-like graph (but not limited to it).

\begin{figure}[h] 
	\centering
	\includegraphics[width=0.66\textwidth]{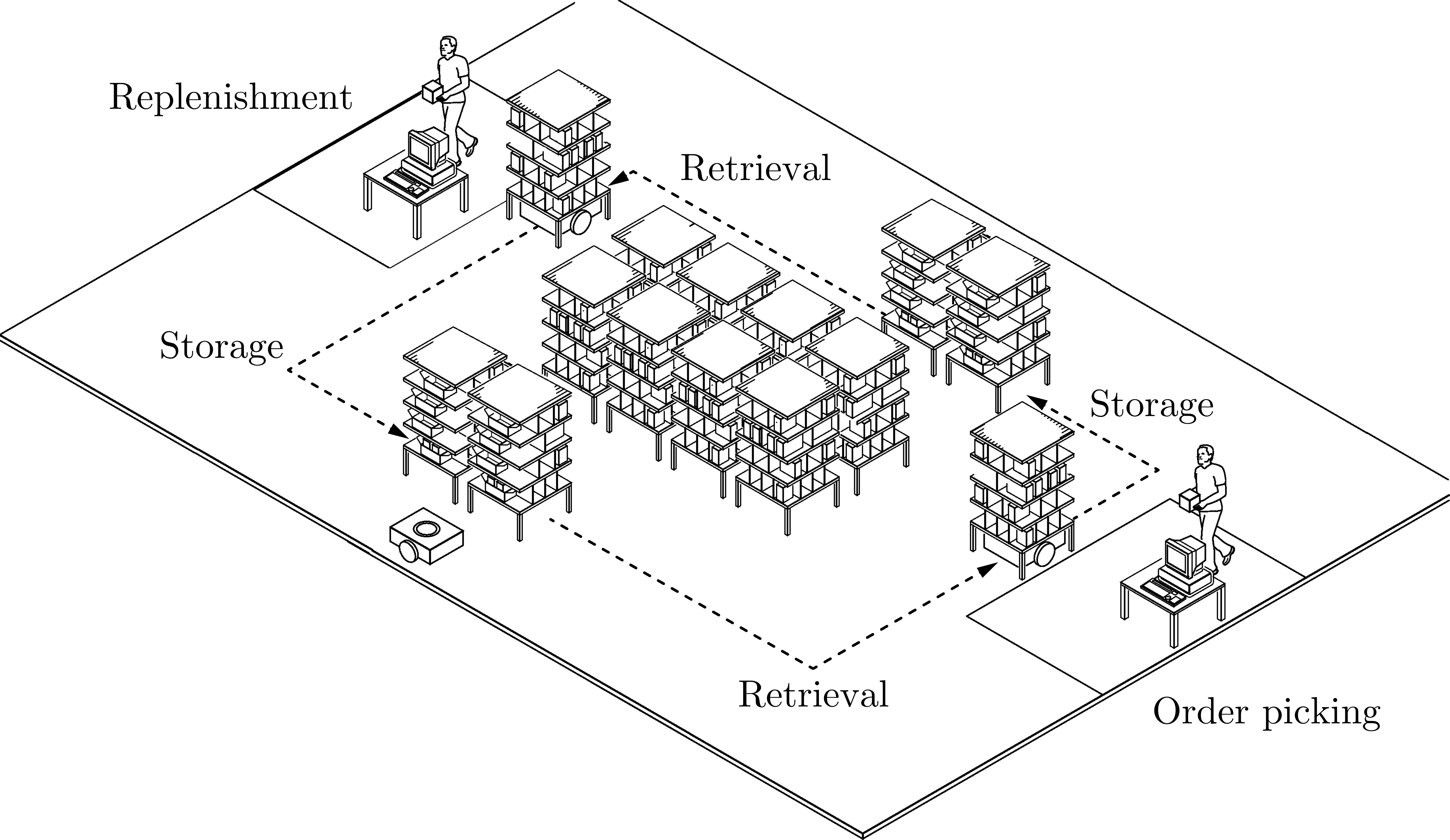}
	\caption{A basic layout of a RMFS based on the Kiva patent \cite{Hoffman.2013}}
	\label{fig:pp_exampleprobleminstance}
	\setlength{\belowcaptionskip}{-10ex}
\end{figure}

In the control of such a system many different decision problems need to be overcome to sustain an overall efficiency (see \cite{Enright.2011}). Limiting the view only at the core of the operational decision problems we propose the abstract structure given by Figure \ref{fig:pp_problemcomponentsoverview} to support a better understanding of the problem interactions in a RMFS. As new item bundles shall be stored in the inventory a \textit{replenishment assignment} controller needs to choose from which replenishment station these shall be inserted and a \textit{bundle storage assignment} controller needs to determine a suitable pod to store it on. For a new customer order first only a pick station to process it is chosen by an \textit{order assignment} controller. The former process leads to insertion requests, while the latter leads to extraction requests. These can now be combined to tasks that shall be executed by a robot. While insert tasks can simply be determined by combining all requests with matching station and pod to a task, we need to select a pod to fulfill the extraction requests using a \textit{pick pod selection} controller first to create extract tasks. The decision about the right pod to bring to a pick station is postponed, because it allows us to exploit more information that becomes available over time. In addition to the insertion and extraction requests that resemble the systems main purpose, park requests are the result of pods that need to be brought back to the inventory. For these a \textit{pod storage assignment} controller determines a suitable storage location. The combination of requests to tasks and the allocation of them to robots is done by a \textit{task allocation} controller. Most of the mentioned tasks require the robot to go from one location to another, which leads to trips from one waypoint in the system to another. For these trips a \textit{path planning} controller is needed to ensure time-efficient, collision- and deadlock-free paths. As the last is the main focus of this work, fixed controllers are used for all other problem components to enable a working system while allowing a fair comparison of the path planning methods (see Section \ref{sec:pp_Simulation}). The existing researches of each decision problem can be found in \cite{Boysen.2017} and \cite{Merschformann.2017b}. We concentrate in this work only the path planning algorithms.

\begin{figure}[h]
	\centering
	\includegraphics[width=0.55\textwidth]{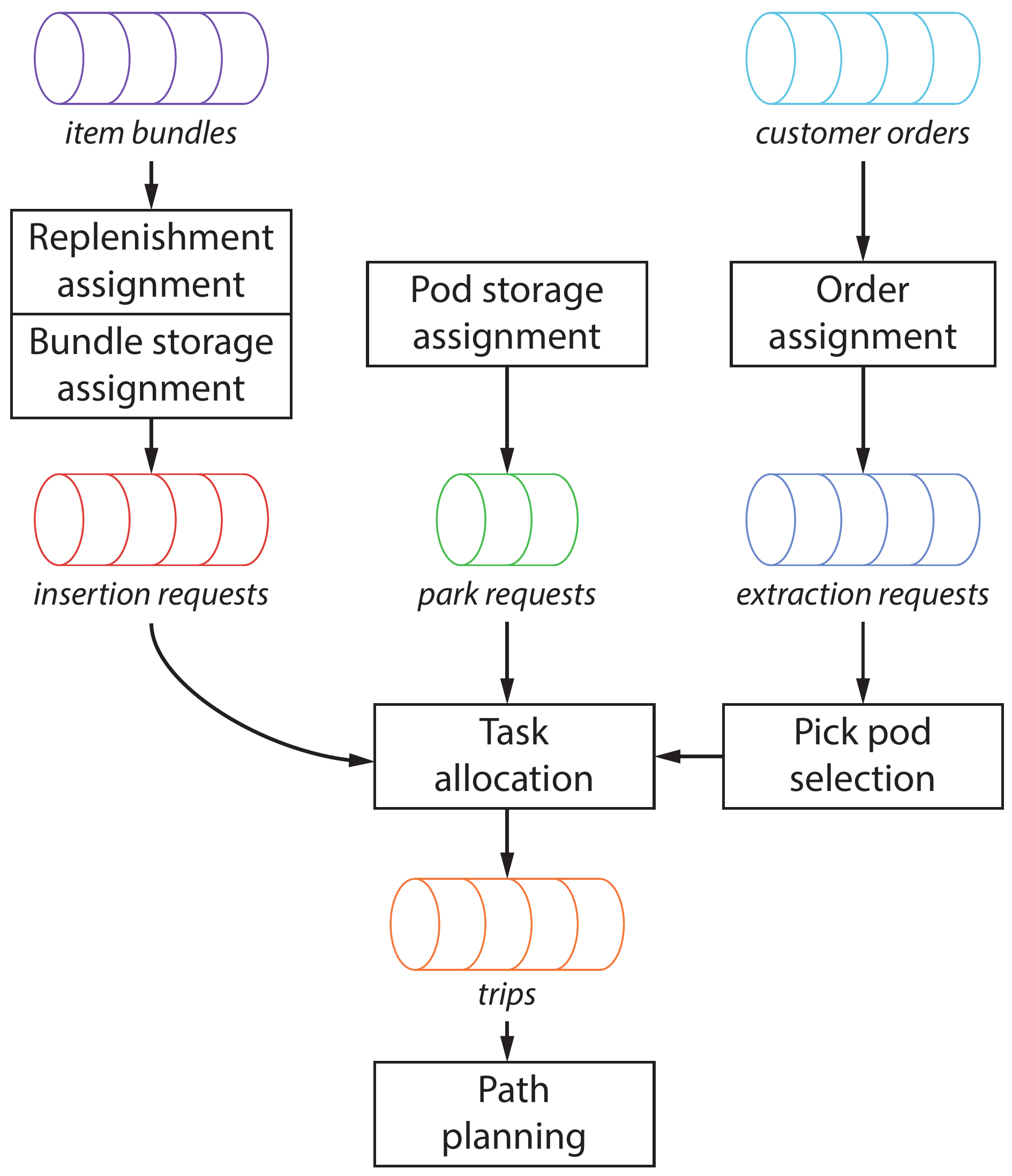}
	\caption{Overview of the core decision problems of a RMFS at operational scope}
	\label{fig:pp_problemcomponentsoverview}
	\setlength{\belowcaptionskip}{-10ex}
\end{figure}

\subsection{MAPFWR}
As mentioned in Section 1, we formulate path planning as the MAPFWR problem, which will be described in this section.
\subsubsection{Problem input}
The main resources are given by the sets of pods $\mIndexBucket \in \mSetBuckets$ and stations $\mIndexStation \in \mSetStations := \mSetIStations \cup \mSetOStations$ (replenishment- and pick-stations). The agents in the described system are mobile robots $\mIndexBot \in \mSetBots$. Therefore, robots and agents are used synonymously in this work. These robots use a multi-layer planar graph $\mSetGraph$ composed of waypoints $\mIndexWaypoint \in \mSetWaypoints$ as its vertices and edges $\mIndexEdge \in \mSetEdges$ connecting them. The edges of the transpose graph $\mSetGraphReverse$ are denoted by $\mSetEdgesReverse$. 
Every layer of the graph represents one tier $\mIndexTier \in \mSetTiers$ of an instance. All positions across the tiers are denoted by $\mVarPos{i}{\mIndexTime} \in \bigcup\limits_{\mIndexTier \in \mSetTiers} \left\{ \mIndexTier \right\} \times \left[ 0,\mParLength{\mIndexTier} \right] \times \left[ 0,\mParWidth{\mIndexTier} \right]$ (i.e. the current tier the robot is on, the x- and the y-coordinate) for movable entities $i\in \mSetBots \cup \mSetBuckets$ (robots and pods) depending on the current time $\mIndexTime \in \mSetTime$ and bounded by the length $\mParLength{\mIndexTier}$ and width $\mParWidth{\mIndexTier}$ of the respective tier $\mIndexTier$. The time-independent position $\mVarPosFixed{i}$ of immovable entities $i\in \mSetWaypoints \cup \mSetStations$ is defined analogously. At this, the time-horizon is continuous ($\mSetTime := \mContinuousGeZero$). At the starting time $\mIndexTime_0$ every robot, pod and station is located at an initial position $\mParWaypointInitial{i} \in \mSetWaypoints$. The time-dependent distance between two entities $i$ and $j$ is defined as the euclidean norm and denoted by $\mGetDistance{i}{j}{\mIndexTime}$ (analogously the time-independent distance is $\mGetDistanceTimeIndependent{i}{j}$). Additionally, the speed of the robots is expressed by $\mVarSpeed{\mIndexBot}{\mIndexTime} \in \mContinuousGeZero$ and orientation by $\mVarOrientation{\mIndexBot}{\mIndexTime} \in \left\{ 0, 2\pi \right\}$. The shape of the robot is abstracted by a circle with radius $\mParRadius{\mIndexBot} \in \mContinuousPositive$. The angular velocity the robot is turning by at time $\mIndexTime$ is indicated by $\mVarAngularVelocity{\mIndexBot}{\mIndexTime} \in \mContinuous$. Furthermore, every robot has a constant acceleration rate $\mParAcceleration{\mIndexBot} \in \mContinuousPositive$, deceleration rate $\mParDeceleration{\mIndexBot} \in \mContinuousPositive$ and maximal velocity $\mParMaxVelocity{\mIndexBot} \in \mContinuousPositive$. The robot can change its orientation clockwise and counterclockwise by the maximal angular velocity of $\mParMaxAngularVelocity{\mIndexBot} \in \mContinuousPositive$. This work abstracts from the rotational momentum.
 
The multiple tiers of an instance can be connected by elevators $\mIndexElevator \in \mSetElevators$, which are able to transport a robot between two waypoints $\mIndexWaypoint$ and $\mIndexWaypoint'$ in the constant time $\mParTimeUseElevator{\mIndexElevator}{\mIndexWaypoint}{\mIndexWaypoint'}$, which depends on the height between the two waypoints to cover and the elevator's speed. During this time all waypoints associated with the elevator are blocked. When a robot stops at a station the time for handling one item is assumed to be constant. We distinguish the time for handling one physical item to fulfill an order at a pick station $\mParTimeOStationHandleUnit{\mIndexStation} \in \mContinuousGeZero$ with $\mIndexStation \in \mSetOStations$ and the time to handle an incoming bundle of items at a replenishment station $\mParTimeIStationHandleUnit{\mIndexStation} \in \mContinuousGeZero$ with $\mIndexStation \in \mSetIStations$. The time for picking up and setting down a pod is given by the constant value $\mParTimePickupSetdown$. Analogous to the robot a pod's shape is abstracted by a circle with the radius $\mParRadius{\mIndexBucket}$. A pod not currently carried by a robot can be stored at one of the parking positions $\mIndexWaypoint \in \mSetBucketParkingSpots \subset \mSetWaypoints$. The choice of the parking position is determined by another planning component that is not discussed any further in this work.
% Removed sentence:
% The function $\mDecisionBucketStorageAssignment{\mIndexBucket}{\mIndexTime} = \mIndexWaypoint$ denotes the assignment of a pod $b$ to a parking position at time $\mIndexTime$.

The length of the edges of the graph is set a priori and has a lower bound determined by the maximal radii of the robots and pods (see Eq. \ref{eq:pp_edgeLength}).
\begin{equation} \label{eq:pp_edgeLength}
\forall \left( \mIndexWaypoint_1, \mIndexWaypoint_2 \right) \in \mSetEdges: \mGetDistanceTimeIndependent{\mIndexWaypoint_1}{\mIndexWaypoint_2} \ge \max_{i\in \mSetBots \cup \mSetBuckets} \max_{i'\in \mSetBots \cup \mSetBuckets \setminus \left\{ i \right\}} \mParRadius{i} + \mParRadius{i'}
\end{equation}
With this requirement, all robots and pods can be arbitrarily distributed among the waypoints.

During runtime a robot can be asked to bring a pod to a station (to insert an item or extract it), to park a pod or to rest at a certain waypoint (these are called as tasks). Every one of these tasks can be decomposed into subtasks as shown in Tab. \ref{tab:pp_task-decomposition}. Note that the first two subtasks of an insert or extract task can be skipped, if the robot already carries the right pod at the time the task is assigned to it. All subtasks, except for the move subtask, block the robot at its current position for the constant period of time described above. Only the move subtask has to be immediately considered here as it requires the generation of a path respecting the movement of other robots in the system. This path starts at the current location of the robot and ends at the destination waypoint specified by the move subtask. %Every time a robot completes a task it is immediately assigned another task by the remaining planning mechanisms, which may involve a new move subtask for which it requires the calculation of a path to the new destination waypoint.
%\subsubsection{Action}

\begin{table}[h]
	\caption{The task types and their subtask types}
	\begin{tabular}{ll}\label{tab:pp_task-decomposition}
		Task type & Subtask types \\
		\hline
		\ Insert & $ \left( \mSubTaskMove, \mSubTaskPickup, \mSubTaskMove, \mSubTaskPut \right)  $ \\
		\ Extract & $ \left( \mSubTaskMove, \mSubTaskPickup, \mSubTaskMove, \mSubTaskGet \right) $ \\
		\ Park & $ \left( \mSubTaskMove, \mSubTaskSetdown \right) $ \\
		\ Rest & $ \left( \mSubTaskMove \right) $ \\
	\end{tabular}
\end{table}

\subsubsection{Constraints during path generation}

In the following, constraints that have to be adhered during path generation are described. At first, robots always have to be located either on a node or an edge (see Eq. \ref{eq:pp_robot-enforce-path-or-edge}).
\begin{align}\label{eq:pp_robot-enforce-path-or-edge}
\forall \mIndexBot \in \mSetBots, \mIndexTime \in \mSetTime: \exists \mIndexWaypoint \in \left\{ \mSetWaypoints \mid \mVarPosTierFixed{\mIndexWaypoint} = \mVarPosTier{\mIndexBot}{\mIndexTime} \right\}: \mVarPosFixedShort{\mIndexWaypoint} = \mVarPosShort{\mIndexBot}{\mIndexTime} \vee & \nonumber\\
\exists \left( \mIndexWaypoint_1, \mIndexWaypoint_2 \right) \in \left\{ \mSetEdges \mid \mVarPosTierFixed{\mIndexWaypoint_1} = \mVarPosTierFixed{\mIndexWaypoint_2} = \mVarPosTier{\mIndexBot}{\mIndexTime} \right\}:\nonumber\\ \mGetDistance{\mIndexWaypoint_1}{\mIndexBot}{\mIndexTime} + \mGetDistance{\mIndexBot}{\mIndexWaypoint_2}{\mIndexTime} = \mGetDistanceTimeIndependent{\mIndexWaypoint_1}{\mIndexWaypoint_2} &
\end{align}
Next, the robots cannot overlap with each other at any time (see Eq. \ref{eq:pp_robot-enforce-no-collision}) as this would immediately result in a collision. This is analogously defined for the pods in Eq. \ref{eq:pp_bucket-enforce-no-collision}.
\begin{equation}\label{eq:pp_robot-enforce-no-collision}
\forall \mIndexBot_1, \mIndexBot_2 \in \mSetBots, \mIndexTime \in \mSetTime: \mVarPosTier{\mIndexBot_1}{\mIndexTime} = \mVarPosTier{\mIndexBot_2}{\mIndexTime} \implies \mGetDistance{\mIndexBot_1}{\mIndexBot_2}{\mIndexTime} \ge \mParRadius{\mIndexBot_1} + \mParRadius{\mIndexBot_2}
\end{equation}
\begin{equation}\label{eq:pp_bucket-enforce-no-collision}
\forall \mIndexBucket_1, \mIndexBucket_2 \in \mSetBuckets, \mIndexTime \in \mSetTime: \mVarPosTier{\mIndexBucket_1}{\mIndexTime} = \mVarPosTier{\mIndexBucket_2}{\mIndexTime} \implies \mGetDistance{\mIndexBucket_1}{\mIndexBucket_2}{\mIndexTime} \ge \mParRadius{\mIndexBucket_1} + \mParRadius{\mIndexBucket_2}
\end{equation}
The function $\mGetBucketOwnerDefinition$ determines the current owner of a pod, which can be either a robot carrying it or a storage location it is stored at. For this we also require the injectivity of the function, because a pod can only be carried or stored by one entity (see Eq. \ref{eq:pp_bucket-enforce-distinct-owner}).
\begin{equation}\label{eq:pp_bucket-enforce-distinct-owner}
\forall \mIndexBucket_1, \mIndexBucket_2 \in \mSetBuckets, \mIndexTime \in \mSetTime: \mGetBucketOwner{\mIndexBucket_1}{\mIndexTime} = \mGetBucketOwner{\mIndexBucket_2}{\mIndexTime} \implies \mIndexBucket_1 = \mIndexBucket_2
\end{equation}
Using this constraint, the position of the pod can be inferred by the position of either the robot or the storage location (see Eq. \ref{eq:pp_bucket-position-inference}).
\begin{align}\label{eq:pp_bucket-position-inference}
\forall \mIndexBucket \in \mSetBuckets, \mIndexTime \in \mSetTime: \left( \exists \mIndexBot \in \mSetBots: \mGetBucketOwner{\mIndexBucket}{\mIndexTime} = \mIndexBot \implies \mVarPosShort{\mIndexBucket}{\mIndexTime} = \mVarPosShort{\mIndexBot}{\mIndexTime} \right) \nonumber\\
\vee \left( \exists \mIndexWaypoint \in \mSetBucketParkingSpots: \mGetBucketOwner{\mIndexBucket}{\mIndexTime} = \mIndexWaypoint \implies \mVarPosShort{\mIndexBucket}{\mIndexTime} = \mVarPosShort{\mIndexWaypoint}{\mIndexTime} \right)
\end{align}
This also infers that a robot carrying a pod cannot pass waypoints at which a pod is stored. Conversely, a robot not carrying one can move beneath pods, thus, using waypoints at which a pod is stored. Furthermore, the turning of a robot is only allowed on a waypoint while it is not moving (see Eq. \ref{eq:pp_bot-only-turn-on-waypoints}). Hence, a robot moving along an edge has to be oriented towards the same direction as the edge (see Eq. \ref{eq:pp_bot-face-edge-orientation-while-moving}). These constraints limit the movement of the robots to the graph similarly to MAPF, but with the difference that a robot can be located between two waypoints at a time $\mIndexTime$.
\begin{align}\label{eq:pp_bot-only-turn-on-waypoints}
\forall \mIndexBot \in \mSetBots, \mIndexTime \in \mSetTime: \mVarAngularVelocity{\mIndexBot}{\mIndexTime} \neq 0 \implies \mVarSpeed{\mIndexBot}{\mIndexTime} = 0 \wedge \exists \mIndexWaypoint \in \mSetBucketParkingSpots: \mVarPosFixedShort{\mIndexWaypoint} = \mVarPosShort{\mIndexBot}{\mIndexTime}
% removed (was confusing and maybe not even completely correct): \text{ with } \mVarAngularVelocity{\mIndexBot}{\mIndexTime} \in \left\{ 0, \mParMaxAngularVelocity{\mIndexBot} \right\}
\end{align}
\begin{align}\label{eq:pp_bot-face-edge-orientation-while-moving}
\forall \mIndexBot \in \mSetBots, \mIndexTime \in \mSetTime,  \left( \mIndexWaypoint_1, \mIndexWaypoint_2 \right) \in \left\{ \mSetEdges \mid \mVarPosTierFixed{\mIndexWaypoint_1} = \mVarPosTierFixed{\mIndexWaypoint_2} = \mVarPosTier{\mIndexBot}{\mIndexTime} \right\}:\nonumber\\ \mVarPosShort{\mIndexBot}{\mIndexTime} \neq \mVarPosFixedShort{\mIndexWaypoint_1} \wedge \mVarPosShort{\mIndexBot}{\mIndexTime} \neq \mVarPosFixedShort{\mIndexWaypoint_2} \wedge \mGetDistance{\mIndexWaypoint_1}{\mIndexBot}{\mIndexTime} + \mGetDistance{\mIndexBot}{\mIndexWaypoint_2}{\mIndexTime} = \mGetDistanceTimeIndependent{\mIndexWaypoint_1}{\mIndexWaypoint_2} \nonumber\\ \implies \mVarOrientation{\mIndexBot}{\mIndexTime} = \arctantwo \left( \mVarPosYFixed{\mIndexWaypoint_1} - \mVarPosY{\mIndexBot}{\mIndexTime}, \mVarPosXFixed{\mIndexWaypoint_1} - \mVarPosX{\mIndexBot}{\mIndexTime} \right) \nonumber\\ \vee \mVarOrientation{\mIndexBot}{\mIndexTime} = \arctantwo \left( \mVarPosYFixed{\mIndexWaypoint_2} - \mVarPosY{\mIndexBot}{\mIndexTime}, \mVarPosXFixed{\mIndexWaypoint_2} - \mVarPosX{\mIndexBot}{\mIndexTime} \right)
\end{align}
Under these principles the movement of a robot can be described as a sequence of an optional rotation, an acceleration, an optional top-speed and a deceleration phase. The time it takes the robot to rotate depends on the rotation angle $\mValAngle$ and is defined in Eq. \ref{eq:pp_rotation-time}.
\begin{equation}\label{eq:pp_rotation-time}
\mGetTimeRotationDefinition \text{ with } \mGetTimeRotationDefinitionAppendix
\end{equation}
Respecting the robots constant acceleration rate and the following equation expresses the time that is consumed while a robot is moving straight (according to \cite{Richard.2008}). At this, two cases have to be distinguished. In the first case the driving distance $d$ is sufficiently long to fully accelerate the robot. This is the case, if the following is true (see Eq. \ref{eq:pp_robot-min-distance-to-fully-accelerate}).
\begin{equation}\label{eq:pp_robot-min-distance-to-fully-accelerate}
d \ge \tilde{d} \text{ with } \mGetTimeDriveDefinitionMinFullAccelerationDistance{d}
\end{equation}
This leads to the time being a sum of the acceleration phase, full speed phase and deceleration phase time. If the distance is too short to fully accelerate, the time only depends on the acceleration and deceleration phases. Equation \ref{eq:pp_robot-drive-time} comprises both cases.
\begin{equation}\label{eq:pp_robot-drive-time}
\mGetTimeDriveDefinition{d}
\end{equation}
Every replenishment station, pick station and elevator can only be used by one robot at a time. For storing bundles at an replenishment station $\mIndexStation$ the robot is blocked for a constant time $\mParTimeIStationHandleUnit{\mIndexStation}$ during the execution of subtask $\mSubTaskPut$ at the station's position. Analogously, the robot waits a constant time $\mParTimeOStationHandleUnit{\mIndexStation}$ during the subtask $\mSubTaskGet$ at pick stations. For traveling with elevator $\mIndexElevator$ from $\mIndexWaypoint$ to $\mIndexWaypoint'$ the robot is blocked for $\mParTimeUseElevator{\mIndexElevator}{\mIndexWaypoint}{\mIndexWaypoint'}$ time-units at $\mIndexWaypoint$ with $\mIndexWaypoint$ and $\mIndexWaypoint'$ contained in $\mSetElevatorWaypoints{\mIndexElevator}$. During execution of the system every robot that has just finished a task requests a new one. The decision of which task the robot has to execute next is determined by another component and given by $\mDecisionTaskAllocation{\mIndexBot}{\mIndexTime}$ with all possible tasks denoted by Tab. \ref{tab:pp_task-decomposition}. For all robots with the subtask $\mSubTaskMove$, a path must be planned. A path $\mValPathIndexed$ of robot $\mIndexBot$ consists of a sequence of triples $\left( \mIndexWaypoint, stop, wait \right)$. The triple describes the actions of the robot with $\mIndexWaypoint \in \mSetWaypoints$ as the waypoint to go to next, $stop \in \left\{ true, false \right\}$ denoting whether the robot stops at the waypoint and $wait \in \mContinuousGeZero$ determining the time the robot waits at the waypoint if $stop = true$ before executing the next action. Let $\mValPathIndexed' \subseteq \mValPathIndexed$ be a subpath for which the first and last triples denote $stop = true$ and for all others $stop = false$. Let $\left( \mIndexWaypoint_1, \dots, \mIndexWaypoint_n \right)$ be the node sequence of the subpath $\mValPathIndexed'$, then two subsequent nodes have to be connected by an edge and all edges have to be parallel to each other. Equations \ref{eq:pp_drive-path-connected} and \ref{eq:pp_drive-path-straight} comprise this requirement.
\begin{equation}\label{eq:pp_drive-path-connected}
\forall i=0, \dots, n-1: \left( \mIndexWaypoint_i, \mIndexWaypoint_{i+1} \right) \in \mSetEdges
\end{equation}
\begin{equation}\label{eq:pp_drive-path-straight}
\begin{split}
\forall i=0, \dots, n-2 : \arctantwo\left( \mVarPosX{\mIndexWaypoint_{i}}{\mIndexTime} - \mVarPosY{\mIndexWaypoint_{i+1}}{\mIndexTime}, \mVarPosY{\mIndexWaypoint_{i}}{\mIndexTime} - \mVarPosX{\mIndexWaypoint_{i+1}}{\mIndexTime} \right)\\
= \arctantwo\left( \mVarPosY{\mIndexWaypoint_{i+1}}{\mIndexTime} - \mVarPosY{\mIndexWaypoint_{i+2}}{\mIndexTime}, \mVarPosX{\mIndexWaypoint_{i+1}}{\mIndexTime} - \mVarPosX{\mIndexWaypoint_{i+2}}{\mIndexTime} \right)
\end{split}
\end{equation}
Table \ref{tab:pp_subtask-transition-conditions} shows the different conditions that need to be valid for completing a subtask for robot $\mIndexBot$ at time $\mIndexTime$. A change of the value for function $\mGetBucketOwner{\mIndexBucket}{\mIndexTime}$ is thereby only possible at time $\mIndexTime = \mIndexTime'$ at which either the subtask $\mSubTaskPickup$ or $\mSubTaskSetdown$ is completed. The most relevant subtask is depicted by $\mSubTaskMove$, because it depicts the operation of going from one waypoint to another. Furthermore, the subtask $\mSubTaskPickup$ lifts a pod such that the robot can carry it. The reverse operation of setting down a pod is depicted by $\mSubTaskSetdown$. At last, the subtask $\mSubTaskPut$ stores a number of item bundles in a pod while the subtask $\mSubTaskGet$ picks a number of items from a pod.
\begin{table}[h]
	\caption{Conditions for switching from the respective subtask to a succeeding one.}
	\begin{tabular}{l|l}\label{tab:pp_subtask-transition-conditions}
		Subtask & Condition \\
		\hline
		\ $\mSubTaskMove$ & $\mVarPosShort{\mIndexBot}{\mIndexTime'} = \mVarPosFixedShort{\mIndexWaypoint} \wedge \mVarSpeed{\mIndexBot}{\mIndexTime'} = 0 $ \\
		\ $\mSubTaskPickup$ & $\exists \mIndexWaypoint \in \mSetBucketParkingSpots \, \forall \mIndexTime \in \left[ \mIndexTime' - \mParTimePickupSetdown, \mIndexTime' \right]: \mGetBucketOwner{\mIndexBucket}{\mIndexTime} = \mIndexWaypoint \wedge \mVarPosShort{\mIndexBot}{\mIndexTime} = \mVarPosShort{\mIndexBucket}{\mIndexTime} = \mVarPosFixedShort{\mIndexWaypoint} $ \\
		\ $\mSubTaskSetdown$ & $\exists \mIndexBucket \in \mSetBuckets \, \forall \mIndexTime \in \left[ \mIndexTime' - \mParTimePickupSetdown, \mIndexTime' \right]: \mGetBucketOwner{\mIndexBucket}{\mIndexTime} = \mIndexBot \wedge \mVarPosShort{\mIndexBot}{\mIndexTime} = \mVarPosShort{\mIndexBucket}{\mIndexTime} = \mVarPosFixedShort{\mIndexWaypoint} $ \\
		\ $\mSubTaskPut$ & $\forall \mIndexTime \in \left[ \mIndexTime' - \mParTimeOStationHandleUnit{\mIndexStation}, \mIndexTime' \right]: \mGetBucketOwner{\mIndexBucket}{\mIndexTime} = \mIndexBot \wedge \mVarPosShort{\mIndexBot}{\mIndexTime} = \mVarPosShort{\mIndexBucket}{\mIndexTime} = \mVarPosFixedShort{\mIndexStation}$ \\
		\ $\mSubTaskGet$ & $\forall \mIndexTime \in \left[ \mIndexTime' - \mParTimeIStationHandleUnit{\mIndexStation}, \mIndexTime' \right]: \mGetBucketOwner{\mIndexBucket}{\mIndexTime} = \mIndexBot \wedge \mVarPosShort{\mIndexBot}{\mIndexTime} = \mVarPosShort{\mIndexBucket}{\mIndexTime} = \mVarPosFixedShort{\mIndexStation}$ \\
	\end{tabular}
\end{table}

%For clarity's sake we define functions $\mGetBotIsCarryingDefinition$ and $\mGetBotIsMoveTaskDefinition$ that indicate whether the respective robot carries a pod and whether it is assigned a move task. 

\subsubsection{Objectives}
One of the main objectives for an RMFS is a high throughput of customer orders. Since we are also considering replenishment operations, we combine the number of picked items with the number of stored bundles to our main metric of handled units overall. This is supported by the robots bringing suitable pods to the stations to complete the aforementioned insert and extract requests. Hence, more time-efficient paths for the robots increase the system's throughput by decreasing the waiting times at the stations. Note that we assume that the subtasks $\mSubTaskPickup$, $\mSubTaskSetdown$, $\mSubTaskGet$ and $\mSubTaskPut$ underlie constant times, therefore the minimization can be achieved by the time-efficient paths, except for possible queuing time at stations. It is similar to the function \textit{sum-of-cost} in \cite{Felner.2004}, where driving and waiting times are aggregated for the finding paths. Therefore, we are also especially interested in the average time for completing a trip. At last, the average distance covered per trip is important when considering wear of the robots and as an indicator for energy consumption. This is also similar to the fuel function in \cite{Felner.2004}. 

\subsection{Related work} \label{subsec:pp_relatedProblem}
The similar problem MAPF is widely discussed in the literature, with applications from video games (see \cite{Krontiris.2013}) to exploration of three-dimensional environments with quadrotor drones (see \cite{Turpin.2014}). In this problem, a set of agents is given, and each of them has its start and end positions. It aims at finding the path for each agent without causing collisions. In many cases minimizing the sum of the timesteps that are required for every agent to reach its goal is also considered as an additional goal. \cite{Sharon.2015} describe this system as generally consisting of a set of agents, each of which has a unique start state and a unique goal state. Moreover, the time is discretized into timepoints and the time for rotation of each agent is set to zero. Also, it is assumed that the time for crossing each arc is constant. These differ to our problem, that means our problem MAPFWR considers continious timepoints and each agent requires time to rotate and the time for crossing each arc might be differed for each agent. Therefore, MAPF can be considered as a special case of our problem. 

In the literature, there are a number of sub-optimal/optimal MAPF solvers (see Sections \ref{subsubsec:pp_suboptimalAlgo} and \ref{subsubsec:pp_optimalAlgo}). In this paper, we modify some of them to solve MAPFWR, especially search-based optimal/sub-optimal MAPF solvers, since they are ususally designed for the sum-of-costs objective function and they support the massive search. The details of the modifications can be found in Section 4. The state-of-the-art MAPF path planning algorithms can be found in \cite{Sharon.2015}.

\subsubsection{Sub-optimal algorithms}\label{subsubsec:pp_suboptimalAlgo}
Finding an optimal solution for the MAPF problem is proven to be NP-hard (see \cite{Yu.}). The standard admissible algorithm for solving the MAPF problem is the A*-algorithm, which uses the following problem representation. A state is a $n$-tuple of grid locations, one for each of n agents. The standard algorithm considers the moves of all agents simultaneously at a timestep, so each state potentially has $b^n$ legal operators ($b$ is the number of possible actions).  With the increasing number of agents the state space grows exponentially, therefore, there are many sub-optimal solvers in the literature for solving this problem more quickly. %Most of them are based on variations of the A*-algorithm, which modify either the search space or the procedure of the A*-algorithm.
\paragraph{Search-based Solvers} These solvers are usually designed for the sum-of-cost objective function. Hierarchical Cooperative A* (HCA*), introduced by \cite{Silver.2005}, plans the agents one at a time according to some predefined orders. A global space time reservation table is used to reserve the path for each agent. That ensures the path for actual agent do not include collisions with the path of previous agent. Windowed-HCA* (WHCA*) in \cite{Silver.2005} enhances HCA* to apply the reservation table within a limit time window. There are some other enhancements of HCA* shown in \cite{Bnaya.2014}, \cite{Sturtevant.} and \cite{Ryan.2008}. \cite{Geramifard.2006} introduce a simple meta-algorithm called Biased Cost Pathfinding (BCP), which assigns a priority to each agent. For each agent, A* is used to find the optimal path without considering collisions. This algorithm extends $h(n)$ with an additional virtual cost to control collisions. 
\paragraph{Rule-based Solvers} These solvers include specific agent-movement rules for different scenarios, but they are usually do not include massive search. Therefore, they are not adopted in this paper. These solvers include TASS (\cite{Khorshid.}), Push-and-Swap (\cite{Luna.2011} and \cite{Sajid.}), Push-and-Rotate (\cite{Wilde.}), BIBOX (\cite{Surynek.}) and diBOX (\cite{Botea.}). %The limitation of TASS \cite{Khorshid.2011} is that this algorithm is complete only for tree graphs, while Push and Swap of \cite{Luna.2011} is only complete if a maximum of $|V|-2$ agents are considered in the graph (see \cite{de.2013}).
\paragraph{Hybrid Solvers} These solvers include both movement rules and massive search. For example, Flow Annotation Replanning (FAR), introduced by \cite{Wang.2008} and \cite{Wang.2011}, adds a flow restriction to limit the movement along a given row or column to only one direction, which avoids head-to-head collisions. The shortest path for each agent is implemented by the A*-algorithm independently, and a heuristic procedure is used to repair plans locally, if deadlocks occur. This method shows similar results to WHCA*, but with lower memory capacities and shorter runtime. A similar approach was proposed in \cite{Jansen.}.
%ACO \cite{Dorigo.1992} and its similar approach DM \cite{Jansen.2008} are also not adopted to solve MAPFWR, since longer paths are generated compared with WHCA*.
%\todo{ACO: better / more reliable justification than the following seems necessary!? - also: is ACO really suitable? it does not handle collisions and does not solve the MAPF originally, do we have a reference to someone who tried to solve the MAPF using ACO?:} 

\subsubsection{Optimal algorithms}\label{subsubsec:pp_optimalAlgo}
There are some optimal solvers for MAPF discussed in the literature, such as the algorithm of \cite{Standley.2010} and Conflict Based Search (CBS) introduced by \cite{Sharon.2015}. It is still possible to find the optimal solutions for real-world problems within an acceptable time: for example, if the paths found by an A*-based algorithm do not contain any collisions. Moreover, it is easier to find optimal solutions if the number of agents is small compared to the size of the graph.
\paragraph{Reduction-based Solvers} Some optimal solvers in the literature try to reduce MAPF to standard known problems, such as SAT (\cite{Gorbenko.2012} and \cite{Surynek.2012}), ASP (\cite{Erdem.}) and CSP (\cite{Ryan.}), since they have existing high-quality solvers. However, they were not designed to solve the sum-of-cost objective function; therefore, it is not easy to modify them for our purpose.  Still, A first reduction-based SAT solver for sum-of-costs objective function was introduced in \cite{Surynek.2016}.%Therefore, our problem cannot immediately be formulated as SAT, such as in \cite{Gorbenko.2012}. Additionally, our problem cannot be solved by hybrid solver MAPP presented by \cite{Wang.2011}, because MAPP is only proven to be complete for graphs which have the slidable property. The instances we use to solve MAPFWR do not fulfill this property, because the alternate connectivity restriction cannot be guaranteed and is also subject to change as a result of changing the locations of pods in between executions of the path planning algorithm. 
\paragraph{Search-based Solvers} Some search-based solvers are based on A* algorithm and try to overcome the drawbacks of A*, such as large open-list of successors and large number of neighbors during branching. In order to reduce the number of nodes, which are generated but never expanded, Standley introduced an operator decomposition (OD). An intermediate node is introduced to ensure only the moves of a single agent are considered when a regular A* node is expanded. The authors in \cite{Goldenberg.2014} introduce another method, called a priori domain knowledge to reduce the number of nodes. Moverover, Standley introduced the Independence Detection (ID) framework to reduce the effective number of agents. The idea behind this is to detect independent groups of agents. Two groups of agents are independent, if an optimal solution can be found by a low-level solver (such as OD) for each group and no conflict occurs between them. Initially, each agent is a group and an optimal solution is found for it. The agents, which have conflicts with each other, are merged into one group and new solutions are found for them, and so on. This method ends if one solution without collision is found, or all agents belonging to one group. This algorithm is called OD\&ID in the rest of this paper. Another optimal MAPF solver based on A* is M* (see \cite{Wagner.2015}). There are some search-based solvers are not based on A* algorithm, such as Conflict Based Search (CBS) introduced by \cite{Sharon.2015} and Increasing Cost Treee Search (ICTS) by \cite{Sharon.}. The idea of CBS is similar to that of the branch-and-bound algorithm (see \cite{Land.1960}), where MAPF is decomposed into a large number of constrained single-agent pathfinding problems. It aims at finding a minimum-cost constraint tree without collisions. This algorithm works on two levels, namely high level and low level. At the high level, conflicts are found and constraints are added, while an optimal path is found for each agent at the low level, which is consistent with the new constraints. The low-level and high-level searches are best-first searches; however, \cite{Barer.2014} introduced Enhanced Conflict-Based Search (ECBS), which uses focal searches (see \cite{Pearl.1984}) for both levels. This algorithm is sub-optimal, but with shorter runtime compared with CBS, since the focal search considers a subset of the best search and expands a node with $f(n) \leq \epsilon f_{min}$. The parameter $\epsilon$ is defined by the user. Cohen et al. \cite{Cohen.2015} combine \textit{highways} with ECBS for solving the Kiva system. According to the authors, the combination has been demonstrated to decrease computational runtime and costs compared with ECBS; moreover, ECBS alone is slow for a large number of agents in the Kiva system. ICTS works with discrete increasing cost tree at high level, while at the low level a non-conflicing complete solution is found. We do not adopt ICTS since MAPFWR works in continuous enviornments. More about search-based optimal solvers for MAPF can be found in \cite{Felner.}.

	\section{Search space} \label{sec:pp_SearchSpace}
	In this section we discuss the properties of search space $\mSetSearchSpace$ as a state space. Note that $\mSetSearchSpace$ is different from the graph $\mSetGraph$. The state in search space $\mSetSearchSpace$ is noted as $\mIndexState$, while the node in Graph $\mSetGraph$ is noted as $\mIndexWaypoint$. As described in Section \ref{subsec:pp_relatedProblem} for MAPF, there are four possible actions for agents in the grid graph and one action for waiting. Therefore, the grade of the search space is $O(5^k)$, where $k$ is the number of agents. However, the grade of the search space of MAPFWR can be infinite, since each agent can get any orientation and can wait for any interval. Therefore, we assume that the waiting time is limited to a given interval $\mParTimeWaiting \in \mContinuousPositive$. So the cost of the arc $(\mIndexState_1,\mIndexState_2)$ for a waiting agent $r$ is $c^W_r(\mIndexState_1,\mIndexState_2)=\mParTimeWaiting$, while the cost of the arc for a moving agent is $c^M_r(\mIndexState_1,\mIndexState_2) = \mGetTimeRotation{\mIndexBot}{\mValAngle}{\mParMaxAngularVelocity{\mIndexBot}} + \mGetTimeDrive{\mIndexBot}{s}{\mParAcceleration{\mIndexBot}}{\mParMaxVelocity{\mIndexBot}}{\mParDeceleration{\mIndexBot}}$. At this, the times for rotation and driving are considered, and $s$ is the Euclidean distance.

With the discreted waiting time, the grade of the search space in MAPFWR is similar to the grade of the search space in MAPF, but the time for each action differs from one agent to another. An example illustrated in Fig. \ref{fig:pp_examplereservation} shows the reservation for a path from $\mIndexWaypoint_1$ to $\mIndexWaypoint_3$ with waiting time 0 at $\mIndexWaypoint_2$. The marked area (blue) means that agents cannot go through this node for the given time, e.g., because another agent already has an ongoing reservation for the node at the time. For MAPF in the left-hand graph, two actions are required, namely ``move to $\mIndexWaypoint_2$'' and ``move to $\mIndexWaypoint_3$''. However, there are two possible cases for MAPFWR. For the case in the central graph (where a quick stop at $\mIndexWaypoint_2$ is done) we get the path: $(\mIndexWaypoint_1,true,0),(\mIndexWaypoint_2,true,0),(\mIndexWaypoint_3,true,0)$. And for the case in the right-hand graph (without stopping at $\mIndexWaypoint_2$) we get the path: $(\mIndexWaypoint_1,true,0),(\mIndexWaypoint_2,false,0),(\mIndexWaypoint_3,true,0)$. $\mIndexWaypoint_2$ in the central graph is blocked for a longer time due to the times for acceleration and deceleration necessary for the stop; instead, the reservation of $\mIndexWaypoint_2$ in the right-hand graph is possible without overlapping. The same problem can also appear in the reservation for a sequence of nodes.
\begin{figure}[h]
	\centering
	\begin{subfigure}[b]{0.75\textwidth}
		\includegraphics[width=\textwidth]{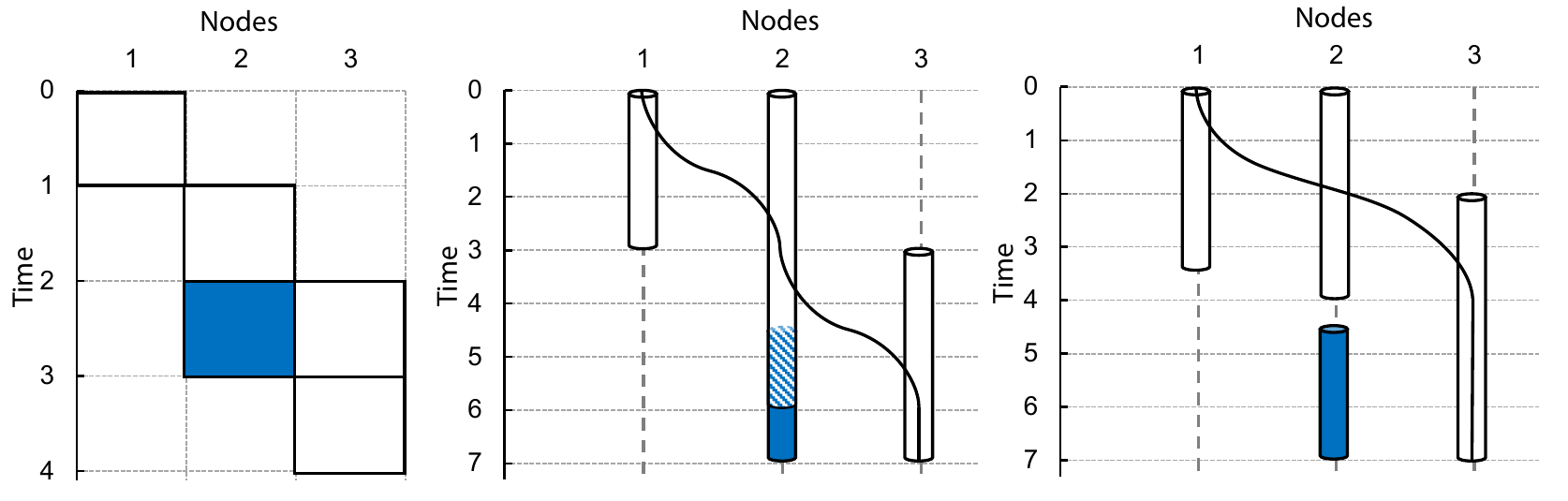}
	\end{subfigure}
	\begin{subfigure}[b]{0.15\textwidth}
		\includegraphics[width=\textwidth]{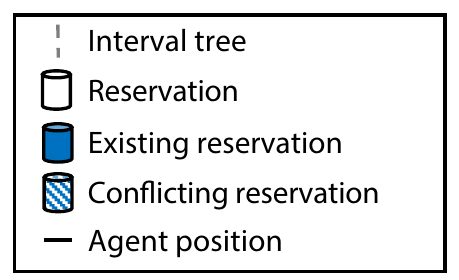}
	\end{subfigure}
	\caption{The reservation for the path from $\mIndexWaypoint_1$ to $\mIndexWaypoint_3$ with waiting time 0 at $\mIndexWaypoint_2$. From left to right: MAPF, MAPFWR with stopping at $\mIndexWaypoint_2$, MAPFWR without stopping at $\mIndexWaypoint_2$}
	\label{fig:pp_examplereservation}
\end{figure}

Therefore, we provide two solutions. The first solution is that we generate a move action only for connected nodes and a wait action with a fixed given time. Once a conflict occurs, we search the move actions in the same direction until we get the first node with a move action without conflict, or no further nodes exist in that direction. Therefore, the maximum grade of states in the search space for an agent is equal to the grade of nodes in the graph. The second solution is extending reservations to the next node in the same direction by assuming a continuous drive. Fig. \ref{fig:pp_costs} shows an example to generate the successors for the node $\mIndexWaypoint_4$, namely $\mIndexWaypoint_5$, $\mIndexWaypoint_6$, $\mIndexWaypoint_7$. If we ignore the wait actions, the number of nodes corresponds to the number of states. The cost for $\mIndexWaypoint_5$ is equal to the sum of the drive time until $\mIndexWaypoint_4$, the time for rotating by the angle $\beta$ and the moving time between $\mIndexWaypoint_4$ and $\mIndexWaypoint_5$. This is analogously done for $\mIndexWaypoint_7$ with the angle $\gamma$. For the calculation of the cost of $\mIndexWaypoint_6$, a backward search is until the last rotation is done (in this example: $\mIndexWaypoint_1$). We can infer that the agent had to stop at this node. Thus, the cost of $\mIndexWaypoint_6$ is determined by adding up the sum of cost up to $\mIndexWaypoint_1$, the time for rotating by the angle $\alpha$, and the drive time from $\mIndexWaypoint_1$ to $\mIndexWaypoint_6$ without stopping. The backward search can be done in $O(1)$, if we keep track of the last rotation or waiting action for each node.

\begin{figure}[h]
	\centering
	\includegraphics[width=0.75\textwidth]{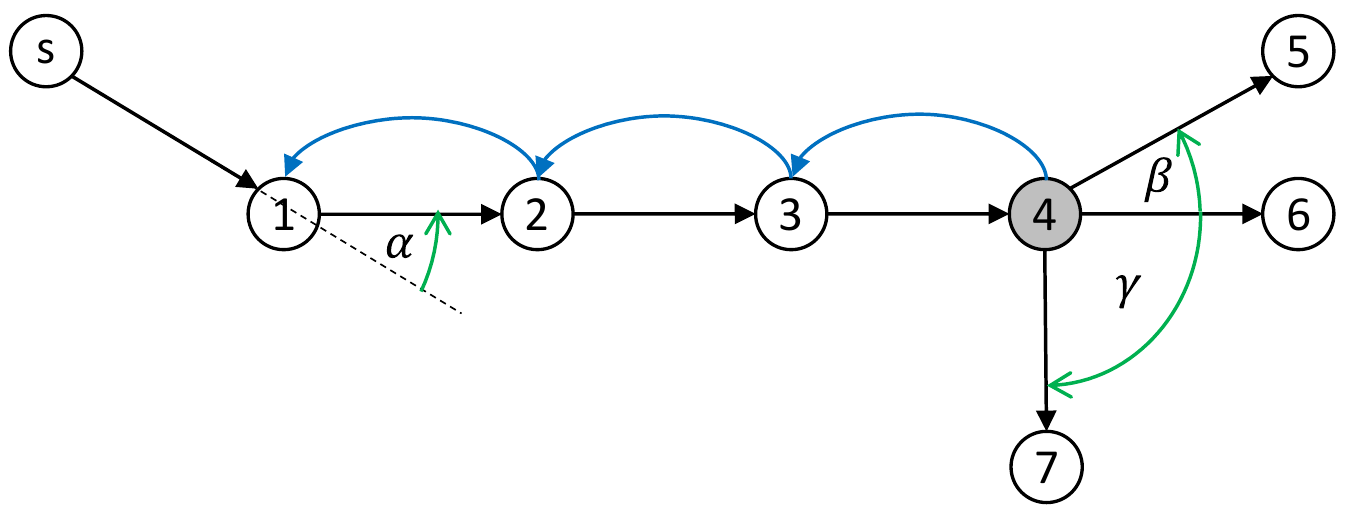}
	\caption{One example for the calculation for the costs of the successors of $\mIndexWaypoint_4$}
	\label{fig:pp_costs}
\end{figure}

In the second solution above, we consider only times (costs), so we can abandon the wait actions. Moreover, we don't need to generate and check the reservations. In this case, the grade of the search space in our problem is equal to the grade of search space in MAPF.

Now we have to prove that the A*-Algorithm is complete and admissible in the search space of our problem, because some algorithms we discuss in the next section are based on the A*-Algorithm. An algorithm is complete if it terminates with a solution in case one exists, while an algorithm is admissible if it is guaranteed to return an optimal solution whenever a solution exists. According to the properties shown in \cite{Pearl.1984}, when the search space is a tree, there is a single beginning state and a set of goal states, the cost of the path is the sum of the costs of all arcs in this path, and the heuristic function $h$ meets the following conditions: $h(\mIndexState)\geq  0  \ \forall n$ and $h(\mIndexState)=0$ if $n$ is a goal state. Moreover, each state should have a finite number of successors and the cost $c(\mIndexState_1,\mIndexState_2)$ of each arc $(\mIndexState_1,\mIndexState_2)$ in the search space should meet the condition $c(\mIndexState_1,\mIndexState_2) \geq \delta \ge 0$. According to the problem description in Section \ref{sec:pp_Problem} and this section, the search space of our problem fulfills most of the conditions, except the last one. Therefore, we only have to prove that the cost of each arc in $S$ has a finite lower bound $\delta$ (see the proof in Section \ref{sec:pp_Theorem}). So the A*-Algorithm is complete and admissible in the search space of our problem.

In the next section, we discuss the path planning algorithms; most of them use the A*-algorithm for searching a path from actual location $\mIndexWaypoint_0$ to goal location $\mIndexWaypoint_e$ for an agent $\mIndexBot$. The heuristic function $h(\mIndexState)$ is defined as:
\[h(\mIndexState)=\mGetTimeDrive{\mIndexBot}{\mGetDistanceTimeIndependent{\mIndexWaypoint_0}{\mIndexWaypoint_e}}{\mParAcceleration{\mIndexBot}}{\mParMaxVelocity{\mIndexBot}}{\mParDeceleration{\mIndexBot}}\]
where $\mGetDistanceTimeIndependent{\mIndexWaypoint_0}{\mIndexWaypoint_e}$ is the Euclidean distance from $\mIndexWaypoint_0$ to $\mIndexWaypoint_e$. Since the moving time increases monotonically during the trip and the Euclidean distance satisfies triangle inequality, hence, the heuristic function $h(\mIndexState)$ is consistent.
%\todo{do we need the following? if we do, we need to fix the reference} According to [Nilsson 1980], a property of A* with a consistent heuristic is that the optimal distance from the start node to a node is known once that node is expanded.
	
	%%%%%%%%%%%%%%%%%%%%%%%%%%%%%%%%%%%%%%%%%%%%%%%%%%%%%%%%%
	%%%%%%%%%%%%%%%%%%%%%%%%%%%%%%%%%%%%%%%%%%%%%%%%%%%%%%%%%
	%%%%%%%%%%%%%%%%%%%%%%%%%%%%%%%%%%%%%%%%%%%%%%%%%%%%%%%%%
	
	\section{Algorithm design}\label{sec:pp_Algo}
	In this section we describe our implemented algorithms, which mainly differ to the original ones for the MAPF with reservation table considering continuous time (see Section \ref{subsec:pp_datasturcture}) and considering kinematic constraints in continuous time. That is followed by the description of our implemented path-planning algorithms in terms of their differences from the existing ones for the MAPF problem, including WHCA*, FAR, BCP, OD\&ID and CBS.  Finally, a method for resolving deadlocks is described.
\subsection{Data structure} \label{subsec:pp_datasturcture}
We describe in this subsection the reservation table, which is introduced in \cite{Silver.2005} for solving the MAPF problem. According to Silver, the impassable space-time regions are marked in the reservation table with the form $(x,y,\mIndexTime)$ and they are stored in a hash table with random keys. Note that $x,y$ are coordinates and $\mIndexTime$ is a timepoint. The left-hand graph in Fig. \ref{fig:pp_reservationsdiscretevscontinuous} illustrates the reservation table used in \cite{Silver.2005}, where the reservations of robot 1 and 2 are marked with green and blue colors. These impassable regions should be avoided during searches of the next robots. The node $\mIndexWaypoint$ is blocked if one robot waits at it. If one robot goes through the arc $(\mIndexWaypoint_1,\mIndexWaypoint_2)$, then its corresponding nodes $\mIndexWaypoint_1$ and $\mIndexWaypoint_2$ are blocked until the robot goes past the arc. Hence, the distance for robots traveling in a convoy is at least one arc.
%, and the robots are allowed to go with the opposite direction. 

We adopt the reservation table to store occupied regions, because it is a sparse data structure, which considers space and time. Since our problem does not consider discrete and constant timepoints, a different implementation compared with \cite{Silver.2005} is illustrated in the right-hand graph of Fig. \ref{fig:pp_reservationsdiscretevscontinuous}. For each node in the graph, one interval tree (see \cite{Cormen.2010}, pp. 350\textendash357) is generated, if a robot waits at it or goes through it (this is also called lazy initialization). Also, a set of intervals is included in each interval tree and each interval has its starting time $\mIndexTime_s$ and ending time $\mIndexTime_e$. The search for one interval $(\mIndexTime_1,\mIndexTime_2)$ at node $\mIndexWaypoint$ and the test of overlapping can be realized using a binary search with runtime $O(\log n)$.

 \begin{figure}[h]
 	\centering
 	\includegraphics[width=0.6\textwidth]{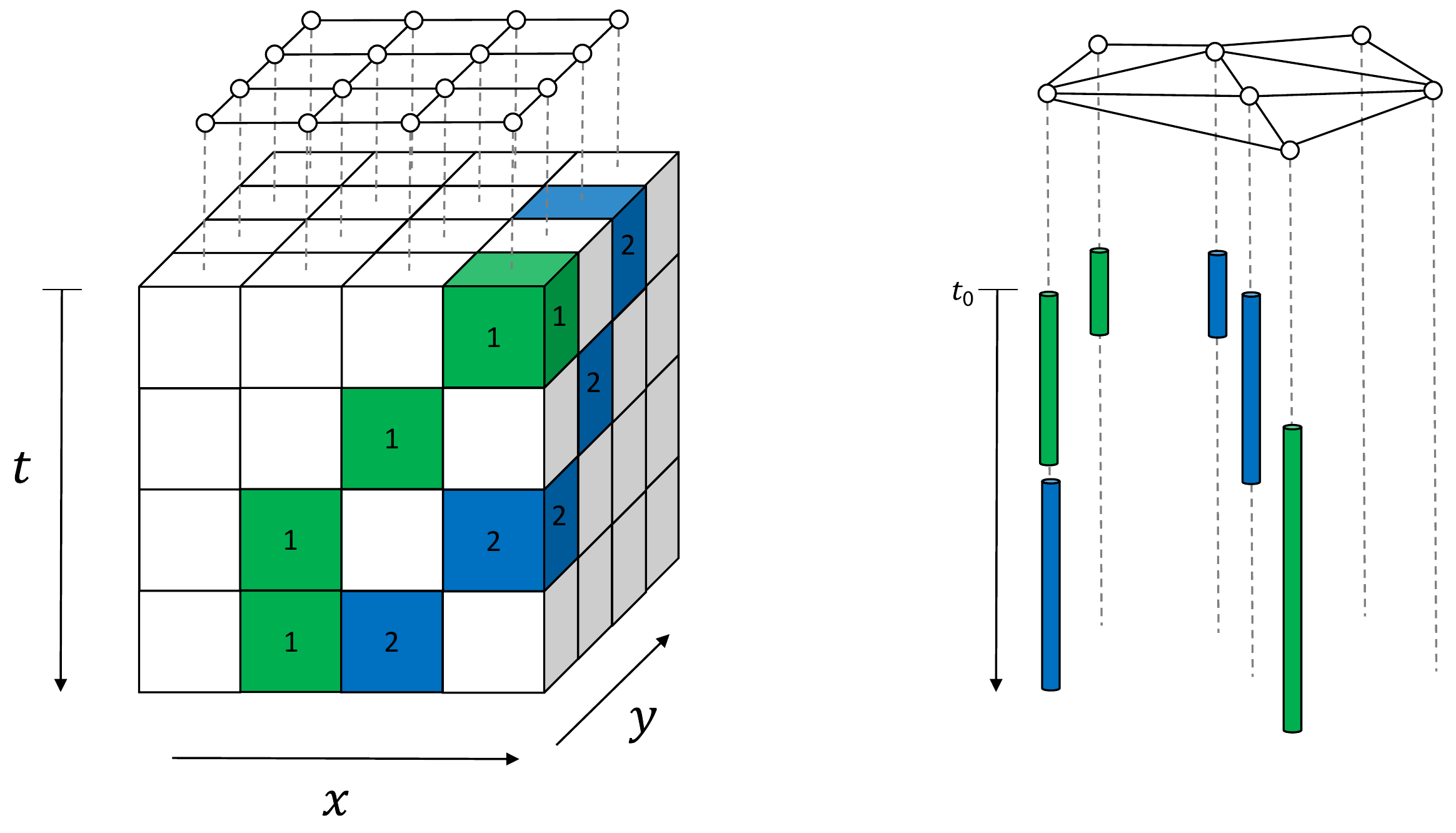}
 	\caption{Data structure for reservations in MAPF of \cite{Silver.2005} vs. MAPFWR path planning}
 	\label{fig:pp_reservationsdiscretevscontinuous}
 \end{figure}
 
%Algorithm \ref{alg:pp_overlap} shows how it works. $\mVarResTab{\mIndexWaypoint}{\mIndexTime_s}{\mIndexTime_e}$ is defined as a reservation table for the node $\mIndexWaypoint$ from the starting time $\mIndexTime_s$ to the ending time $\mIndexTime_e$. 

An example of the reservation for a robot in the reservation table is shown in Fig. \ref{fig:pp_reservationsexample}. The robot begins with node $\mIndexWaypoint_2$ at time $t_0$ and ends with node $\mIndexWaypoint_6$. Once the robot travels between two nodes, both nodes are blocked (vertical bars). Moreover, it takes longer time to speed up and slow down. This data structure is used in all path-planning algorithms in the following subsections.

%\begin{algorithm}[h]
%	\caption{IsIntersectionFree$(rtable,\mIndexWaypoint,t_1,t_2)$}
%	\label{alg:pp_overlap}
%	\DontPrintSemicolon
%	\LinesNumbered
%	$i_l \leftarrow 0$\\
%	$i_r \leftarrow length(rtable[\mIndexWaypoint])$\\
%	\While{$i_l \leq i_r$}{
%		\eIf{$t_1 == rtable_{s}[\mIndexWaypoint][\dfrac{i_l+i_r}{2}]$}{
%			\textbf{return} $false$
%			}
%		{
%			\eIf{$t_1 < rtable_s[\mIndexWaypoint][\dfrac{i_l+i_r}{2}]$}{
%				$i_r \leftarrow \dfrac{i_l+i_r}{2}-1$}{
%				$i_l \leftarrow \dfrac{i_l+i_r}{2}+1$}
%		}
%		}
%	\textbf{return} $(i_l=0$ $or$ $rtable_e[\mIndexWaypoint][i_l-1]\leq t_1)$ $and$ $(i_l=length(rtable[\mIndexWaypoint]) $ $or$ $t_2 \leq rtable_s[\mIndexWaypoint][i_l])$
%\end{algorithm}
\begin{figure}[h]
	\centering
	\includegraphics[width=0.66\textwidth]{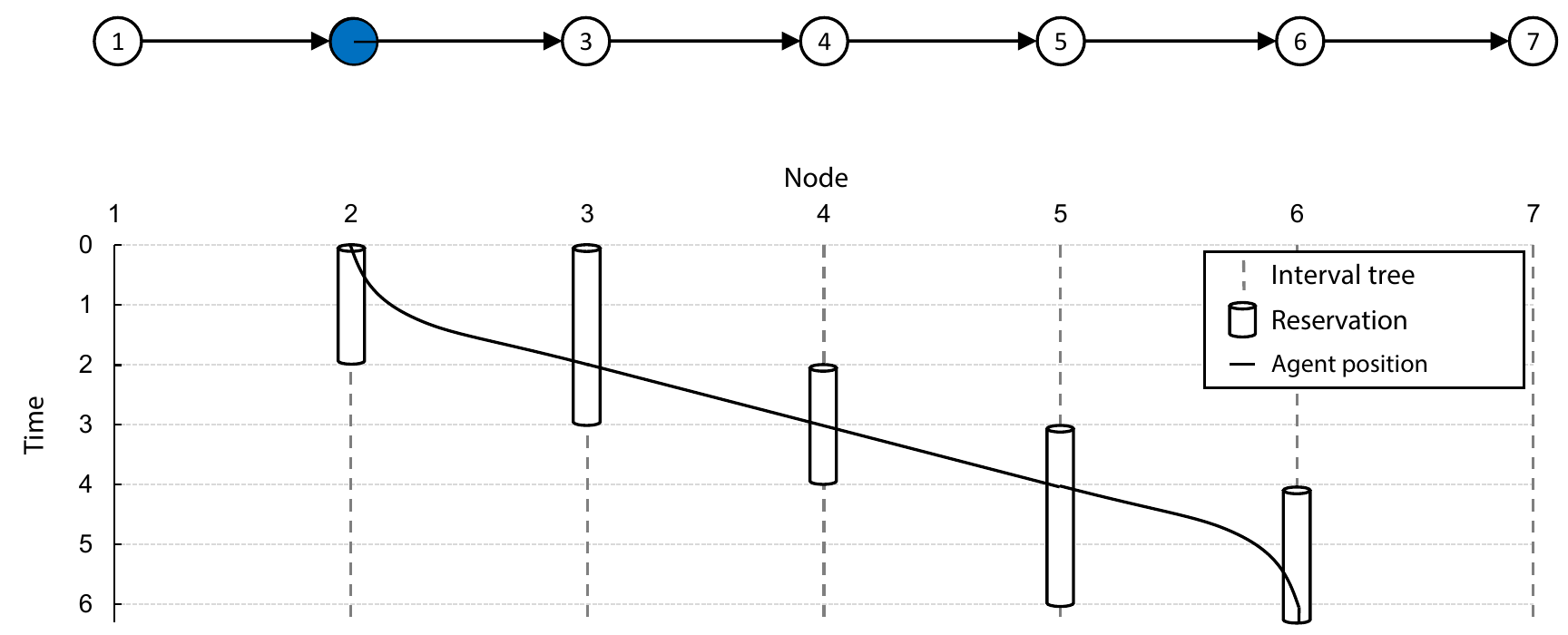}
	\caption{An example in the reservation table for the path from $\mIndexWaypoint_2$ to $\mIndexWaypoint_6$}
	\label{fig:pp_reservationsexample}
\end{figure}

\subsection{WHCA*} \label{sec:pp_WHCA*}
The main difference to the existing WHCA*, applied in MAPF in \cite{Silver.2005}, is the modified reservation table (see Section \ref{subsec:pp_datasturcture}). Moreover, we use another heuristic to calculate costs individually for our problem. We describe below the adjusted WHCA* for the MAPFWR problem.

We have an arbitrary robot $\mIndexBot$, which is located on the node $\mIndexWaypoint_s$ on time $\mIndexTime$ in the graph, and its goal is the node $\mIndexWaypoint_e$. The A*-algorithm finds an optimal path for this robot from $\mIndexWaypoint_s$ to $\mIndexWaypoint_e$, if the heuristic function $h(\mIndexWaypoint_s)$ is admissible. For now, we assume that the nodes in the graph are equal to the states in the search space; moreover, we do not consider the other robots and wait actions. $\mGetAStarCostsEstimatedTwoD{\mIndexWaypoint_s}$ is calculated with Eq. \eqref{eq:pp_drivertime}.\\
\begin{equation}\label{eq:pp_drivertime}
\mGetAStarCostsEstimatedTwoD{\mIndexWaypoint_s} :=\mGetTimeDrive{\mIndexBot}{\mGetDistanceTimeIndependent{\mIndexWaypoint_s}{\mIndexWaypoint_e}}{\mParAcceleration{\mIndexBot}}{\mParMaxVelocity{\mIndexBot}}{\mParDeceleration{\mIndexBot}}
\end{equation}
where the time is estimated by calculating the driving time for the euclidean distance while considering acceleration and deceleration times. Thus, it is impossible for a robot to cover the distance between the two waypoints in shorter time. We use this heuristic function $\mGetAStarCostsEstimatedTwoD{\mIndexWaypoint_e}$ in the Reverse Resumable A* (RRA*) algorithm, which searches the path from $\mIndexWaypoint_e$ to $\mIndexWaypoint_s$ in $\mSetGraphReverse$; therefore the heuristic function $\mGetAStarCostsEstimatedByRRA{\mIndexWaypoint_e}$ is equal to $\mGetAStarCostsEstimatedTwoD{\mIndexWaypoint_e}$. $g^{RRA^*(\mIndexWaypoint_e)}$ is the function that calculates the sum of time for rotation and moving from $\mIndexWaypoint_e$ to $\mIndexWaypoint_s$ like discussed above. This algorithm ends when $\mIndexWaypoint_s$ enters the closed set $C^{RRA^*}$ and we do not consider the rotations of $\mIndexWaypoint_s$ and $\mIndexWaypoint_e$ here. As explained in \cite{Silver.2005}, the reason why we use RRA* in WHCA* is that the result of RRA* provides better assumption for the paths without collisions. It is a lower bound of the solutions found in WHCA*, because the solution found in RRA* is optimal for a single agent. If $\mIndexWaypoint_s$ enters the closed set $C^{\text{\textit{RRA}}^*}$ then the value of the heuristic function is equal to the value of the cost function of RRA* to the goal node; otherwise the search continues until $\mIndexWaypoint_i \in C^{\text{\textit{RRA}}^*}$. The heuristic function $\mGetAStarCostsEstimatedByWHCA{\mIndexWaypoint_s}$ used by WHCA* can now be defined like in Eq. \eqref{eq:pp_hWHCA*}.
\begin{equation} \label{eq:pp_hWHCA*}
\mGetAStarCostsEstimatedByWHCA{\mIndexWaypoint_s} =
\begin{cases}
g^{\text{\textit{RRA}}^*}(\mIndexWaypoint_s) & \text{if } \mIndexWaypoint_s \in C^{\text{\textit{RRA}}^*} \\
\text{\textit{RRA}}^*(\mIndexWaypoint_e,E^{-1},{\mIndexWaypoint_s},g^{\text{\textit{RRA}}^*},h^{\text{\textit{RRA}}^*}) & \text{other cases}
\end{cases}
\end{equation}

In the following, we describe two variants of WHCA*, namely volatile and non-volatile WHCA*. The first one was used in \cite{Silver.2005} with some modifications. However, this variant is time-consuming, since the calculated paths are not stored for the next execution and it is suitable for the case where the moving obstacles change the graph in the next execution, such as in a computer game. In our case, the obstacles are deterministic and predictable. Therefore, we develop the latter, non-volatile WHCA*, in Section \ref{subsubsec:pp_nonvolatileWHCA*}.
\subsubsection{Volatile WHCA*} \label{subsubsec:pp_volatileWHCA*}
Algorithm \ref{alg:pp_WHCA*volatile} shows how volatile WHCA* works. Let $\mSetBuckets' \subset \mSetBuckets$ be a subset of pods, which are not carried by robots, while $\mSetBots' \subset \mSetBots$ is a subset of robots, which have subtask $\mSubTaskMove$. Also, the priority $\mVarPriority{\mIndexBot}$  for each robot $\mIndexBot$ is set to $0$ at the beginning of the algorithm (line \ref{line:pp_whcavinitprio}). The search repeats until a path $\mValPath{\mIndexBot}{\mIndexTime}$ is found for each robot $\mIndexBot$ at time $\mIndexTime$ or the number of iterations $i$ reaches the given iteration limit $\mParIteration$. In each iteration, the reservation table $\mVarResTab$ is initialized with fixed reservations (using the function $\mGetFixedReservation{\mSetBots'}$). These are called ``fixed'', because for each already moving robot $\mIndexBot \in \mSetBots'$ there are nodes that are required to be reserved until $\mIndexBot$ reaches its next planned stop. Then, the robots are sorted based on three criteria (see the sort-function $\mSortAgentwithPriority{\mSetBots'}{\mVarPriority{r}}$). The first one is the priority of the robot, while the second one is to check whether a robot carries a pod. These robots are preferred because they cannot move beneath other pods, i.e. they have fewer paths available to get to their goals. The last criterion is the distance towards the goal, preferring robots nearer to their goal. The set $\mSetBlockedNodes$ contains all blocked nodes. The positions of the robots $\mIndexBot \in \mSetBots \setminus \mSetBots'$ are blocked and stored in the set $\mSetBlockedNodes$. Additionally, if the robot $\mIndexBot$ is carrying a pod (i.e., the function $\mGetBotIsCarrying{\mIndexBot}{\mIndexTime}$ is true), then the nodes all other pods $\mIndexBucket \in \mSetBuckets'$ are stored at are blocked and stored in $\mSetBlockedNodes$ as well. Once the destination of the robot $\mIndexBot$ has been changed, RRA* should be recalculated with the actual set $\mSetBlockedNodes$ (see line \ref{line:pp_whcavrrastarupdateandpriocalc}). Also, wait actions are added based on its priority (through the function $\mGetWaitsteps{\mIndexBot}{\lfloor 2^{\mVarPriority{\mIndexBot}-1} \rfloor}$). After that a new path is expanded using $A^*_{ST}$, which considers space and time (see Section \ref{sec:pp_SearchSpace}). If a path was found ($\pi_{\mIndexBot} \neq \emptyset$), then the reservation table $\mVarResTab$ is updated ($\mAddReservation{\pi_{\mIndexBot}}{\mVarResTab}$), otherwise the priority of this robot $\mIndexBot$ should be increased (see line \ref{line:pp_whcavprioincrease}).

\begin{algorithm}[h]
	\caption{$\text{\textit{WHCA}}_v^*(\mIndexTime,\mIndexWaypoint,\mSetBots',\mSetBuckets',\mParIteration)$}
	\label{alg:pp_WHCA*volatile}
	\DontPrintSemicolon
	\LinesNumbered
	%	$A' \leftarrow \{a|state[a]=move\}$\\
	\lForEach {$\mIndexBot \in \mSetBots'$}{
		$\mVarPriority{r} \leftarrow 0$\label{line:pp_whcavinitprio}
	}
	%	$prio \leftarrow init_with_zero(A')$\\
	$i \leftarrow 1$\\
	\While{$i \leq \mParIteration \vee \exists \mIndexBot \in \mSetBots': \pi_{\mIndexBot}=null$}{
		$\mVarResTab \leftarrow \mGetFixedReservation{\mSetBots'}$,
		$\mSortAgentwithPriority{\mSetBots'}{\mVarPriority{r}}$\\
		\ForEach {$\mIndexBot \in \mSetBots'$}{
			$\mSetBlockedNodes \leftarrow \mGetPosition{\mSetBots \setminus \mSetBots'}$\\
			\lIf{$\mGetBotIsCarrying{\mIndexBot}{\mIndexTime}$}{
				$\mSetBlockedNodes \leftarrow \mSetBlockedNodes \cup \mGetPosition{\mSetBuckets'}$
			}
			\If{$\mGetDestinationIsChanged{\mIndexBot}$}{
				$RRA^*(\mIndexBot,\mSetBlockedNodes)$,$\pi_{\mIndexBot}=\mGetWaitsteps{\mIndexBot}{\lfloor 2^{\mVarPriority{\mIndexBot}-1} \rfloor}$\label{line:pp_whcavrrastarupdateandpriocalc}\\
				$\pi_{\mIndexBot}=\pi_{\mIndexBot} \bigcup A^*_{ST}(\mIndexTime,\mIndexWaypoint,\mIndexBot,\mVarResTab, \mSetBlockedNodes,g^{WHCA^*},h^{WHCA^*})$
			}
			\lIf{$\pi_{\mIndexBot} \neq \emptyset$}{
				$\mAddReservation{\pi_{\mIndexBot}}{\mVarResTab}$
			}
			\lElse{
			$\mVarPriority{r} \leftarrow \mVarPriority{r}+1$\label{line:pp_whcavprioincrease}
			}
	}
	$i \leftarrow i+1$
}
\textbf{return} $\pi$
\end{algorithm}

For each robot, we search a sequence of actions in the search graph to get to the goal without a collision. We assume that a robot stops at node $\mIndexWaypoint_i$ and then we generate the following states in the search graph. They are only generated, if they do not block any robots. Moreover, the states, which are in conflict with existing reservations in the reservation table, cannot be generated as well. But we consider the reservation table only within time window $w$; after that time window, the rest of the paths calculated by $\text{\textit{RRA}}^*$ are used. We search a sequence of actions for each robot iteratively; if we cannot find a sequence of actions for a robot without conflicts, then the priority of that robot is increased and we restart the search.

Fig. \ref{fig:pp_exampleWHCA*} shows an example of evasion, where robot $\mIndexBot_1$ tries to move from $s_1$ to $g_1$, while robot $\mIndexBot_2$ tries to move from $s_2$ to $g_2$. Also, we follow the criteria of sorting that we discussed above. Moreover, we assume in this example that the time to go through an arc is one timepoint. The time window $w$ is 10 in this example. We begin with the robot $\mIndexBot_1$, and we get the following reservation with the form [beginning time, ending time]: [0,1] for $\mIndexWaypoint_3$; [0,2] for $\mIndexWaypoint_2$; [1,3] for $\mIndexWaypoint_1$. For the robot $\mIndexBot_2$ we do not get any reservations without collision, since [0,1] should be reserved for it. Therefore, we increase the priority of $\mIndexBot_2$. Here, we get a deadlock, because each robot tries to reach another's beginning node. To solve this deadlock, the robot with the higher priority should at first choose the wait action. According to line \ref{line:pp_whcavrrastarupdateandpriocalc} in Algorithm \ref{alg:pp_WHCA*volatile}, the length of the wait robot grows exponentially with the priority. In this example, $\mIndexBot_1$ in the fourth iteration gets the priority 1 and $\mIndexBot_2$ gets the priority 2. Therefore, the reservation of $\mIndexBot_2$ is done as follows: [0,2] for $\mIndexWaypoint_2$; [1,3] for node $\mIndexWaypoint_3$; [2,3] for $\mIndexWaypoint_4$. After that, $\mIndexBot_1$ can find a path without collisions, namely the path 3-5-3-2-1. In another case, if we try the robot $\mIndexBot_2$ in the first iteration, then the deadlock can be solved as well.

\begin{figure}[h]
	\centering
	\includegraphics[width=0.35\textwidth]{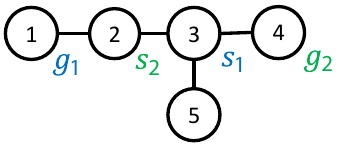}
	\caption{An example of evasion, where robot $\mIndexBot_1$ tries to move from $s_1$ to $g_1$, while robot $\mIndexBot_2$ tries to move from $s_2$ to $g_2$}
	\label{fig:pp_exampleWHCA*}
\end{figure}

\subsubsection{Non-volatile WHCA*} \label{subsubsec:pp_nonvolatileWHCA*}
As mentioned before, the volatile variant in the previous subsection has a problem in that the existing path and reservation for each robot are recalculated in each execution of the algorithm. This occurs already if only one robot with a move action does not have a path. Instead of that, the non-volatile variant of WHCA* stores the existing path and reservation for each robot, which brings a much shorter runtime. However, reusing the existing paths might cause a problem for generating new paths, since we cannot do any modification of existing paths to adopt the new ones. A comparison between the volatile and non-volatile variants of WHCA* can be found in the next section. Alg. \ref{alg:pp_WHCA*nonvolatile} shows the process of the non-volatile variant. The set $\mSetBots' \subset \mSetBots$ is a subset of robots, which has subtask $\mSubTaskMove$ and does not have a path yet. First of all, they are sorted analogously to how they are sorted in the volatile variant (function $\mSortAgent{\mSetBots'}$). Also, the reservation table $\mVarResTab$ is reorganized with the function $\mReorganizeReservationTable{\mVarResTab}$. This means that all future reservations of the robots in need for a new path are dropped (see Section \ref{subsec:pp_datasturcture}). Then, we find a path for each robot iteratively (similar to the volatile variant). Moreover, we check at the end of each path search for the robot $\mIndexBot$ whether a final reservation is possible (function $\mAddFinalReservation{\mIndexBot}$). A final reservation blocks the nodes, where the last action occurs. Hence, a path with only a sequence of wait actions is always possible. By doing this, we can find a solution without prioritizing.

\begin{algorithm}[h]
	\caption{$\text{\textit{WHCA}}_n^*(\mIndexWaypoint,\mSetBots',\mSetBuckets',\mParIteration)$}
	\label{alg:pp_WHCA*nonvolatile}
	\DontPrintSemicolon
	\LinesNumbered
%	$A' \leftarrow \{a|state[a]=move \bigwedge \pi(a)=null\}$\\
	$\mSortAgent{\mSetBots'}$,
	$\mReorganizeReservationTable{\mVarResTab}$\\
	\ForEach{$\mIndexBot \in \mSetBots'$}{
		%$removereservations(a,r)$\\
			$\mSetBlockedNodes \leftarrow \mGetPosition{\mSetBots \setminus \mSetBots'}$\\
			\lIf{$\mGetBotIsCarrying{\mIndexBot}{\mIndexTime}$}{
				$\mSetBlockedNodes \leftarrow \mSetBlockedNodes \cup \mGetPosition{\mSetBuckets'}$
			}
			\lIf{$\mGetDestinationIsChanged{\mIndexBot}$}{
				$RRA^*(\mIndexBot,\mSetBlockedNodes)$
			}
		$\pi_{\mIndexBot}=\pi_{\mIndexBot} \bigcup A^*_{ST}(\mIndexTime,\mIndexWaypoint,\mIndexBot,\mVarResTab, \mSetBlockedNodes,g^{\text{\textit{WHCA}}^*},h^{\text{\textit{WHCA}}^*})$,
		$\mAddReservation{\pi_{\mIndexBot}}{\mVarResTab}$,
		$\mAddFinalReservation{\mVarResTab}$
	}
	\textbf{return} $\pi$
\end{algorithm}

Now we apply this algorithm to the example in Fig. \ref{fig:pp_exampleWHCA*}. The final reservation does not guarantee to find a path for the robot $\mIndexBot_1$, since $\mIndexWaypoint_2$ is blocked for $\mIndexBot_2$. Such conflict might often occur if we consider a large number of robots. Therefore, we increase the heuristic costs of the node, which is on the shortest path of another robot, for robot $\mIndexBot$ (see Eq. \ref{eq:pp_heuristiccostp}). So the robot $\mIndexBot_1$ does not stay at $\mIndexWaypoint_3$ but at $\mIndexWaypoint_5$, since $\mIndexWaypoint_3$ is on the shortest path of $\mIndexBot_2$, and a higher heuristic cost is considered on $\mIndexWaypoint_3$. 
\begin{equation} \label{eq:pp_heuristiccostp}
h_p^{\text{\textit{WHCA}}^*}=h^{\text{\textit{WHCA}}^*}(\mIndexWaypoint_i)+c_p| \{\mIndexBot' \in \mSetBots'| \mIndexBot' \neq \mIndexBot \wedge \mIndexWaypoint_i \in \pi^{\text{\textit{RRA}}^*}(\mIndexBot',\mIndexWaypoint_{\mIndexBot})\}|
\end{equation}

\subsection{FAR}
The FAR algorithm from \cite{Wang.2008} was used to calculate paths in a flow-annotated search graph for the MAPF problem. Each grid graph can be converted to a flow-annotated graph, but some rules should be held to ensure the connectivity of the graph (see \cite{Wang.2008}). The flow-annotated graph is necessary for the FAR algorithm to detect deadlocks and to resolve them. We use a warehouse layout as in \cite{Lamballais.2016}, which contains some properties of a flow-annotated graph (see Section \ref{subsubsec:pp_Layout}). The contribution here is different evasion strategies designed for the MAPFWR, which will be described latter in this section.

Wang and Botea's idea in \cite{Wang.2008} was to use the A*-algorithm to find paths for robots for as long as possible without collisions. They modify the reservation table of \cite{Silver.2005} to reserve a number of steps ahead von the generated paths. Instead, we use the reservation table shown in Section \ref{subsec:pp_datasturcture}. Algorithm \ref{alg:pp_FAR} shows how FAR works. Similarly to WHCA*, the paths of all robots $\mIndexBot \in \mSetBots'$ with subtask $\mSubTaskMove$ are planned iteratively. Recall that $\mSetBuckets'$ is the set of pods, which are not being carried by robots. Initially, the reservation table $\mVarResTab$ stores all existing reservations of all robots to their next nodes and final reservations for these (returning from the function $\mGetFixedAndFinalReservation{\mSetBots'}$). As soon as a path is generated for a robot $r$, the final reservations in $\mVarResTab$ will be deleted (function $\mRemoveFinalReservation{\mIndexBot}{\mVarResTab}$). The algorithm $RRA^*$ is restarted with the actual block nodes in $\mSetBlockedNodes$. From the resulting path the first \textit{Hop} is extracted. The \textit{Hop} is a sequence of actions, which include \textit{rotate}, \textit{move} from a starting node and \textit{stop} at an ending node. This ending node is considered as the next beginning node for the recalculation of the path. This differs from the original idea of FAR, which stops at each node and recalculates the path there. Since we consider the physical properties of robots, it makes more sense to use the \textit{Hop}. The function $\mGetHop{RRA^*(\mIndexBot, \mSetBlockedNodes)}{\mVarResTab}$ returns the \textit{Hop} (line \ref{line:pp_fargethop}). If the \textit{Hop} does cause collisions, it is shortened until it is possible to submit it to the reservation table. If the \textit{Hop} does not contain any nodes, then the returning value $\mIndexBot'$ is the robot who blocks the next node. If no path is found in RRA* ($\pi_\mIndexBot=\emptyset$), then the robot $\mIndexBot'$ is considered as the next $\mIndexBot$. If it is the first time, then the robot $\mIndexBot$ will wait for a given time interval (a wait action is added to the path $\pi_\mIndexBot$ through the function $\mAddWaitAction{\pi_\mIndexBot}$). The relation $(\mIndexBot,\mIndexBot')$ means that a robot $\mIndexBot$ waits for another robot $\mIndexBot'$. Such a relation is established if no \textit{Hop} is found for the robot $\mIndexBot$ and removed again as soon as one is found. If there is a circle of the transitive closure of the relation from $\mIndexBot$ (returning from the function $\mGetRelationIsCircle{\mIndexBot}{\mIndexBot'}$) or the robot $\mIndexBot$ waits more than the given waiting time $\mParTimeWaiting$, then an evasion strategy $\mUseEvasionStrategy{\mIndexBot}$ should be called. The evasion strategy returns an alternative path without wait actions. In the following subsections, two new evasion strategies are described. At the end of the FAR algorithm, the final reservations will be stored in the reservation table again (using function $\mAddFinalReservation{\mVarResTab}$).

\begin{algorithm}
	\caption{$\text{\textit{FAR}}(\mIndexTime,\mIndexWaypoint,\mSetBots',\mSetBuckets',\mParTimeWaiting)$}
	\label{alg:pp_FAR}
	\DontPrintSemicolon
	\LinesNumbered
	%$A' \leftarrow \{a|state[a]=move\}$\\
	$\mVarResTab \leftarrow \mGetFixedAndFinalReservation{\mSetBots'}$\\
	\ForEach {$\mIndexBot \in \mSetBots'$}{
		$\mRemoveFinalReservation{\mIndexBot}{\mVarResTab}$,
		$\mSetBlockedNodes \leftarrow \mGetPosition{\mSetBots \setminus \mSetBots'}$\\
		\lIf{$\mGetBotIsCarrying{\mIndexBot}{\mIndexTime}$}{
			$ \mSetBlockedNodes \leftarrow \mSetBlockedNodes \bigcup \mGetPosition{\mSetBuckets'}$
		}
		$(\pi_\mIndexBot,\mIndexBot') \leftarrow \mGetHop{RRA^*(\mIndexBot, \mSetBlockedNodes)}{\mVarResTab}$\label{line:pp_fargethop}\\
		\If{$\pi_\mIndexBot=\emptyset$}{
			%        	$\mSetRelationWaitfor \leftarrow \mSetRelationWaitfor \cup {(\mIndexBot,\mIndexBot')}$\\
			\lIf{$\mGetLastAction{\pi_\mIndexBot}\neq wait$}{$\mAddWaitAction{\pi_{\mIndexBot}}$}  
			\Else {
				\lIf{$\mGetRelationIsCircle{\mIndexBot}{\mIndexBot'} \vee t > \mParTimeWaiting$}{
					$\mUseEvasionStrategy{\mIndexBot}$}
			}    	
		}
		$\mAddFinalReservation{\mVarResTab}$
		%       $addpathandfinalreservations(\pi_\mIndexBot,\mVarResTab)$
	}	
	\textbf{return} $\pi$
\end{algorithm}

\subsubsection{Evasion strategy 1: rerouting} \label{subsubsec:pp_evasionStrategy1}
In this strategy we try to find a new path for a robot $\mIndexBot$ to escape from a deadlock. The node, at which the next blocked robot stands, is stored to $\mSetBlockedNodes$. Then, a new path is generated again. If another robot is blocked on this new path, then this rerouting starts again. The number of calls for rerouting is limited by a given number. If the number of calls exceeds that given number or no path is found due to the blocked nodes, then this robot $\mIndexBot$ has to wait for a given time period. This method is called $\text{\textit{FAR}}_r$ in the remainder of this paper.

\subsubsection{Evasion strategy 2: evasion step } \label{subsubsec:pp_evasionStrategy2}
Wang and Botea explain that a deadlock occurs if all four nodes of a square are occupied by robots, and each robot wants to exchange nodes with each other. In the flow-annotated graph, each robot has two arcs to leave the actual node. If one arc causes the deadlock with other robots, then another arc should be chosen. In our case, we choose a random arc, which does not cause any deadlocks. After the robot has gone through that selected arc, it should wait for a random time between 0 and $\mParTimeWaiting$. If that arc does not exist, then the robot should simply stay and wait for a given time. This strategy is called $\text{\textit{FAR}}_e$ in the remainder of this paper.

\subsection{BCP}
Geramifard et al. describe BCP as a meta-algorithm for the path planning of each robot. The main difference to the original algorithm is the reservation table shown in Section \ref{subsec:pp_datasturcture}. In the original algorithm, a hash table is used to detect collisions of a discrete time-horizon. In our modified version the reservation table stores continuous time intervals per robot and searches for collisions using binary search. As shown in Algorithm \ref{alg:pp_BCP}, each robot with the subtask $\mSubTaskMove$ is initialized with a heuristic function $h$ as shown in Eq. \eqref{eq:pp_drivertime}. BCP changes the function until a runtime limit $\mParTimeRunning$ is reached or paths without any collisions are found. In each iteration, the reservation table is initialized with fixed reservations. A path is determined by $A^*_S$ for each robot in the two-dimensional search space (without considering time) while considering blocked nodes. If there is no path found for a robot, then the robot has to wait for a given time period (line \ref{line:pp_bcpnopathwait}). As long as a path is found for the robot $\mIndexBot$, the reservations will be stored in the reservation table. If the reservation is not possible, since another robot $\mIndexBot'$ already has a reservation of this node $\mIndexWaypoint$, then the heuristic function will be modified to penalize the value of robot $\mIndexBot$ on $\mIndexWaypoint$. This and causes the search to consider alternative paths. If there are multiple collisions (detected through the function $\mDetectCollision{\mIndexBot}{\mVarResTab}$), then the first one is always chosen by function $\mGetCollision{\pi_\mIndexBot}{\mIndexBot}$. Also, the cost of $h^{BCP}_\mIndexBot$ is updated (function $\mGetCost{c}{h_r^{BCP}}$). If BCP stops due to a reached runtime limit, at least one robot's path contains a collision. This collision is detected during simulation, further execution of the path is stopped and BCP is called again. Hence, we need a lower bound of runtime between two calls, such that the robot can wait at the node. The deadlock can be solved with the method shown in Section \ref{subsec:pp_resolvedeadlock}.
\begin{algorithm}[h]
	\caption{$\text{\textit{BCP}}(\mIndexTime,\mIndexWaypoint,\mSetBots',\mSetBuckets',\mParTimeRunning)$}
	\label{alg:pp_BCP}
	\DontPrintSemicolon
	\LinesNumbered
	%	$A' \leftarrow \{a|state[a]=move\}$\\
	\lForEach {$\mIndexBot \in \mSetBots'$}{
		$h_r^{BCP} \leftarrow h$
	}
	\While{$t \leq \mParTimeRunning$}{
		$\mVarResTab \leftarrow \mGetFixedReservation{\mSetBots'}, c \leftarrow \emptyset$\\
		\ForEach {$\mIndexBot \in \mSetBots'$}{
			$\mSetBlockedNodes \leftarrow \mGetPosition{\mSetBots \setminus \mSetBots'}$\\
			\lIf{$\mGetBotIsCarrying{\mIndexBot}{\mIndexTime}$}{
				$\mSetBlockedNodes \leftarrow \mSetBlockedNodes \cup \mGetPosition{\mSetBuckets'}$
			}
			$\pi_\mIndexBot \leftarrow A^*_S(\mIndexBot,\mSetBlockedNodes,g,h_r^{BCP})$\\
			\lIf{$\pi_\mIndexBot = \emptyset$}{
				$\mAddWaitAction{\pi_\mIndexBot}$\label{line:pp_bcpnopathwait}
			}
			\Else{
				\lIf{$\mDetectCollision{\pi_\mIndexBot}{\mIndexBot}$}
				{
					$c \leftarrow \mGetCollision{\pi_\mIndexBot}{\mIndexBot}$,
					$\mGetCost{c}{h_r^{BCP}}$
					}
			}
		}
		\lIf{$c = \emptyset$}{\textbf{return} $\pi$}
	}
\textbf{return} $\pi$
\end{algorithm}

\subsection{OD\&ID}
When determining the next action for all agents the grade of the search space is $O(b^k)$, where $k$ is the number of robots and $b$ is the number of actions. As mentioned in Section \ref{sec:pp_Problem}, intermediate states are introduced in \cite{Standley.2010} to reduce the number of generated states in this algorithm. We describe in the following subsections the modified operator decomposition and independence detection for MAPFWR.
\subsubsection{Operator decomposition (OD)} 
One of the main modifications of OD for MAPFWR is to consider continuous times. Hence, full-value states cannot be reached in OD. Each action of a robot requires an individual time interval. Therefore, it is possible that from one state to its following state, all robots are located on different timepoints. So we define for our problem that each robot $\mIndexBot$ has an action in state $\mIndexState$ before a time stamp $\mVarTimeStampeOD{\mIndexBot}{\mIndexState}$ is reached. In the initial state $\mIndexState_0$, all robots have a set of empty actions and the time stamp $\mVarTimeStampeOD{\mIndexBot}{0}$ for each robot $\mIndexBot$ is equal to 0. From one state to its following state, the time stamp is increased with the time of the action, which the robot selects, while the time stamps of the other robots remain unchanged. $g(\mIndexState)$ is the sum of times from $\mIndexState_0$ to $\mIndexState$, while $h(\mIndexState)$ is the sum of estimated values for all robots to reach their goal nodes. These values are calculated by RRA* without considering collisions (similar to WHCA* in Section \ref{sec:pp_WHCA*}).

In the initial state, a robot is chosen randomly to generate the next state, since the time stamps of all robots are set to 0. In other cases, the robot with the lowest time stamp is chosen to generate the next state. The following state is generated for each action that does not cause any collisions. The reservation table is also used here to detect collisions. The reservation table is initially empty and gets populated each following state of this robot by the reservations of all other robots that end after $\mVarTimeStampeOD{\mIndexBot}{\mIndexState}$. In order to generate paths without collisions, a final reservation is required after the last reservation. For each possible action of a robot, the resulting reservation is tested for collisions. If there is no collision, then the time stamp is increased correspondingly and the following state is generated. Finally, the reservation table is emptied and the following state is considered. If all states are generated for a robot, then the A*-algorithm searches for the next state without any successors. After this selection, the robot with the lowest time stamp will be selected and the following state is generated, and so on.

The OD is not suitable for real-time problems, therefore, several stopping criteria are defined to limit the runtime. First of all, the OD stops if a state is reached where all robots arrive at their goal nodes. It also means that all paths are free of collisions. Secondly, the OD also stops if a path is found for each robot within a predefined time limit. This is similar to the time limit of WHCA*, but the searches of following states in OD are simultaneous, not iterative. The third stopping criterion is that the maximum number of expanded states is reached. In the second and third stopping criteria, the state is selected from a search space $\mSetSearchSpace' \subset \mSetSearchSpace$, which is defined in Eq. \ref{eq:pp_subsetSearchSpace},
\begin{equation} \label{eq:pp_subsetSearchSpace}
\mSetSearchSpace' = \left\{ \mIndexState \in \mSetSearchSpace \mid \max_{\mIndexBot \in \mSetBots} \mVarTimeStampeOD{\mIndexBot}{\mIndexState} \geq \dfrac{\max_{\mIndexState' \in \mSetSearchSpace}\max_{\mIndexBot \in \mSetBots}\mVarTimeStampeOD{\mIndexBot}{\mIndexState'}}{2}\right\}
\end{equation}
where the set $\mSetSearchSpace'$ includes all states, whose selected paths reach into the second half of the expanded time interval.

For better understanding of OD, the search space of A* with OD is shown in Fig. \ref{fig:pp_ODID} for the same example that is illustrated in Fig. \ref{fig:pp_exampleWHCA*}. Recall that a robot $\mIndexBot_1$ begins with $\mIndexWaypoint_3$ and ends with $\mIndexWaypoint_1$, while a robot $\mIndexBot_2$ begins with $\mIndexWaypoint_2$ and ends with $\mIndexWaypoint_4$. Let $\mParAcceleration{\mIndexBot}$, $\mParDeceleration{\mIndexBot}$ and $\mParMaxVelocity{\mIndexBot}$ be 1 for all $\mIndexBot \in \mSetBots$. Moreover, the length of each arc is also equal to 1. Therefore, the moving time can be calculated as the number of the arcs plus 1. The time of a rotation by 360$^{\circ}$ also is 1. The length of a wait action here is 5. The final reservation is used in this example as well. In the initial state $\mIndexState_0$, the moving time of both robots is 3, and $\mIndexBot_1$ has to rotate by 90$^{\circ}$. Therefore, we get the value of $f(\mIndexState_0)$ $6.25$. There, the time stamps of $\mIndexBot_1$ and $\mIndexBot_2$ are equal to 0, so we can select one of them randomly. In this example we firstly choose $\mIndexBot_1$, which has three possible actions, moving to $\mIndexWaypoint_5$ ($\mIndexState_1$) or $\mIndexWaypoint_4$ ($\mIndexState_2$) or $\mIndexWaypoint_5$ ($\mIndexState_3$). However, $\mIndexBot_1$ cannot reach $\mIndexWaypoint_2$ through the final reservation. Also, the cost of $\mIndexState_2$ is cheaper than $\mIndexState_1$. Thus, $\mIndexState_2$ is selected as the expanded state. In this state, $\mVarTimeStampeOD{\mIndexBot_2}{\mIndexState_2}$ is equal to 0, therefore the actions for $\mIndexBot_2$ will be selected. For this robot, the reservation of the move action on $\mIndexWaypoint_3$ is $[0,2]$ and the final reservation on $\mIndexWaypoint_4$ is $[0,\infty)$. Thus, $\mIndexBot_2$ can either move to $\mIndexWaypoint_1$ or wait. In this example, we can see the influence of the length of the wait action on the solution. If the length is too small, then there are too many states and the second stopping criterion is reached without generating a good solution. On the contrary, the expansion of state $\mIndexState_1$ can happen.

\begin{figure}[h]
	\centering
	\includegraphics[width=0.75\textwidth]{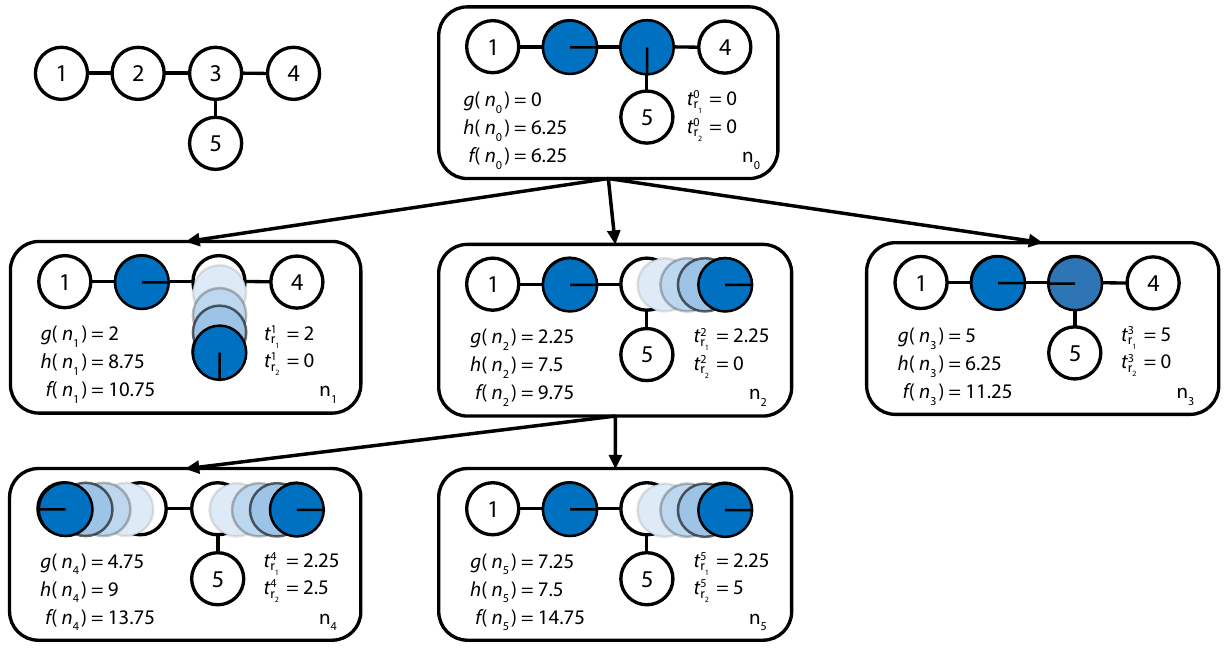}
	\caption{The search space of A* with OD for the case where a robot $\mIndexBot_1$ begins with $\mIndexWaypoint_3$ and ends with $\mIndexWaypoint_1$, while a robot $\mIndexBot_2$ begins with $\mIndexWaypoint_2$ and ends with $\mIndexWaypoint_4$}
	\label{fig:pp_ODID}
\end{figure}
\subsubsection{Independence detection (ID)}
The runtime of OD increases exponentially with the number of robots (see \cite{Standley.2010}). Therefore, Standley introduced ID, which considers disjoint subsets of robots. This means the paths can be planned for each subset of robots independently. Initially, each robot is considered as its own group and for each robot a path is generated with OD. In each iteration, the generated paths of all groups are tested to determine whether they are free of collision. If a collision occurs between two groups, these two groups are merged into one group. Then, new paths are generated for the new groups. This repeats until either all groups can be combined without any collisions or there is only one group remaining.

\subsection{CBS}
The algorithm CBS is introduced by \cite{Sharon.2015}, and uses the constraint tree to dissolve conflicts gradually. Each state in the constraint tree has three properties: one constraint, a solution including a set of paths for all robots and the cost. CBS is a meta-algorithm, which uses the components of a path-planning algorithm to find paths in three dimensional search space, while some nodes are blocked for a given time. Therefore, the algorithm $A^*_{ST}$ is used as in WHCA*. The initial state of the CBS is that a path is found for each robot. There, no restrictions about the conflicts, i.e. collisions, are considered. The corresponding cost of such a state is the sum of costs of all paths of all robots. A state is considered as generated, if a path for each robot is found. Such generated states will be selected for expansion, if they contain conflicts between their paths. I.e., a conflict occurs for a path, if the reservation of a robot $\mIndexBot_1$ for $\mIndexWaypoint$ overlaps with that of a robot $\mIndexBot_2$ for the same node. A new interval for $\mIndexWaypoint$ is calculated based on the maximum ending time and the minimum beginning time of both reservations. Now, two successor states are generated and restricted using this interval. In the first, a new path is searched for robot $\mIndexBot_1$, which is free of overlaps with the new calculated interval. Then, the same is done for robot $\mIndexBot_2$.

It is efficient that the paths and their corresponding costs are calculated once for the initial state, since only the cost of a new path for a robot is updated in each following state, and the paths of other robots remain. The cost of a state is calculated based on a $\delta$-evaluation. And only the new restriction is stored to each following state. In order to hold all restrictions for a state, the tree is traversed from the parent states to the initial state. The same is done to get the path of a state. Recall that, a path for each robot is found in the initial state.

The CBS method is similar to the branch-and-bound algorithm, which includes some selection strategies (see \cite{Land.1960}). The best-first strategy was selected in \cite{Sharon.2015}. Additionally, we apply breadth-first and depth-first strategies to find feasible solutions faster, because MAPFWR has to be solved in a real-time environment. 

Moreover, we use different stop criterium for MAPFWR. The original CBS algorithm stops, if a collision-free solution is found. A solution for MAPF is optimal if an optimal path-planning solver is used and the best-first search is applied as well. However, such a statement is no longer applicable for the CBS algorithm used to solve our problem, since the restrictions include an arbitrary limit to the length of reservations. Moreover, an early termination is required for the application in a real-time environment, because an optimal solution is rarely found for a large number of potential collisions. Therefore, the runtime of this algorithm is limited. It is possible that no feasible solution is found before termination. In this case, we return the solution with the longest time interval until the first collision. Right before the collision would occur CBS is called again to determine a new solution.

\subsection{Resolving deadlocks} \label{subsec:pp_resolvedeadlock}
All the path-planning algorithms we have developed have a runtime limit and timeout is required between two executions. Therefore, it is possible that no path can be found or the path has only wait actions. Moreover, most of the algorithms above, except $FAR_e$, are deterministic, thus, these results are repeated at each call, which might cause an interruption of the system. In order to resolve deadlocks, we keep track of the time each robot approached its current node. If the robot does not leave the node within a given time, then we randomly choose a neighboring node that is not blocked or reserved. If such a node exists, it is assigned to the robot as its next goal. Moreover, it is possible that the robot has only one node to select and the method produces a path that includes the original node again. In this case, both nodes are blocked, and other robots do not have a chance to go through them. Therefore, a random waiting time is required, which is evenly distributed in $[0,\mParTimeWaiting]$. Thereby, the robots, which stand closely to each other, can move apart.

	%%%%%%%%%%%%%%%%%%%%%%%%%%%%%%%%%%%%%%%%%%%%%%%%%%%%%%%%%
	%%%%%%%%%%%%%%%%%%%%%%%%%%%%%%%%%%%%%%%%%%%%%%%%%%%%%%%%%
	%%%%%%%%%%%%%%%%%%%%%%%%%%%%%%%%%%%%%%%%%%%%%%%%%%%%%%%%%
	
	\section{Simulation study}\label{sec:pp_Simulation}
	In this section we first describe our simulation framework while also defining parameters of the experimental setup. Furthermore, we discuss the computational results and compare the different applied methods.

\subsection{Simulation framework}
We use an event-driven agent-based simulation to capture the behavior of RMFS. The 2D- und 3D-visualizations of the simulation are shown on the left and right sides of Fig. \ref{fig:pp_screenshots} respectively, where 3D-visualization is shown for two tiers. The details of the layout will be described later in this section. To investigate the effectiveness in conjunction with the path planning methods described in this work we fix the methods for the other decision problems to the following simple policies. This means that orders are assigned to pick stations randomly (\textit{order assignment}), bundles are assigned to replenishment stations randomly (\textit{replenishment assignment}), bundles are assigned to pods randomly (\textit{bundle storage assignment}), pods are send to random free storage locations (\textit{pod storage assignment}), pods are selected for picking by the number of requests that can be completed with them (\textit{pick pod selection}) and the robots work for all stations equally, but with a preference for pick stations (\textit{task allocation}). As mentioned in Section \ref{subsec:pp_RMFS}, all of these decisions ultimately result in simple requests for the robots to complete; e.g., an order present at a station results in a demand for a certain item. This requires a suitable pod to be brought to this station. Thus, a path from its position to the pod and further on to the station needs to be generated for the robot executing the task.

\begin{figure}[ht]
	\centering
	\begin{subfigure}[b]{0.5\textwidth}
		\centering
		\includegraphics[width = \textwidth]{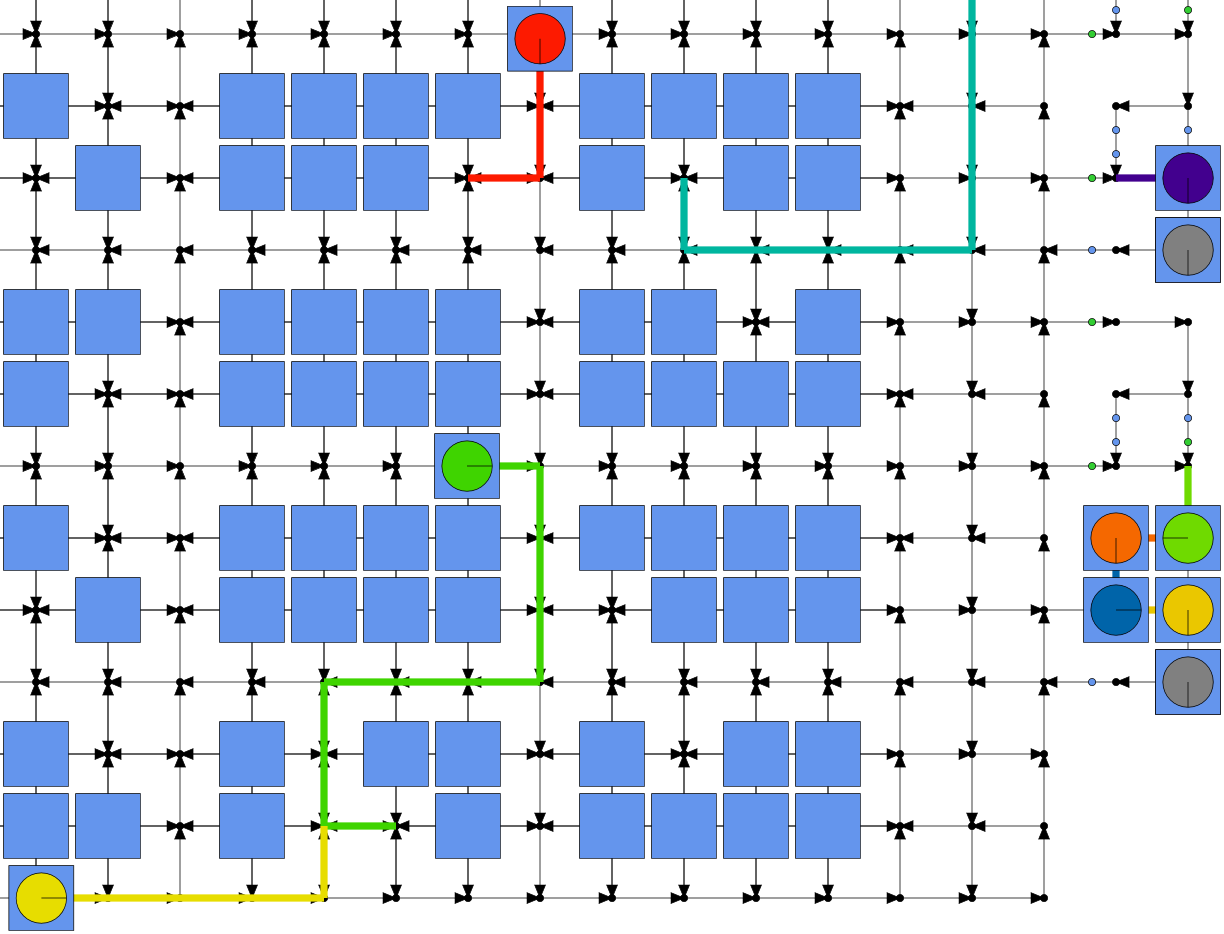}
		\caption{Detailed 2D view}
		\label{fig:pp_simulationscreenshot2D}
	\end{subfigure}
	~
	\begin{subfigure}[b]{0.46\textwidth}
		\centering
		\includegraphics[width = \textwidth]{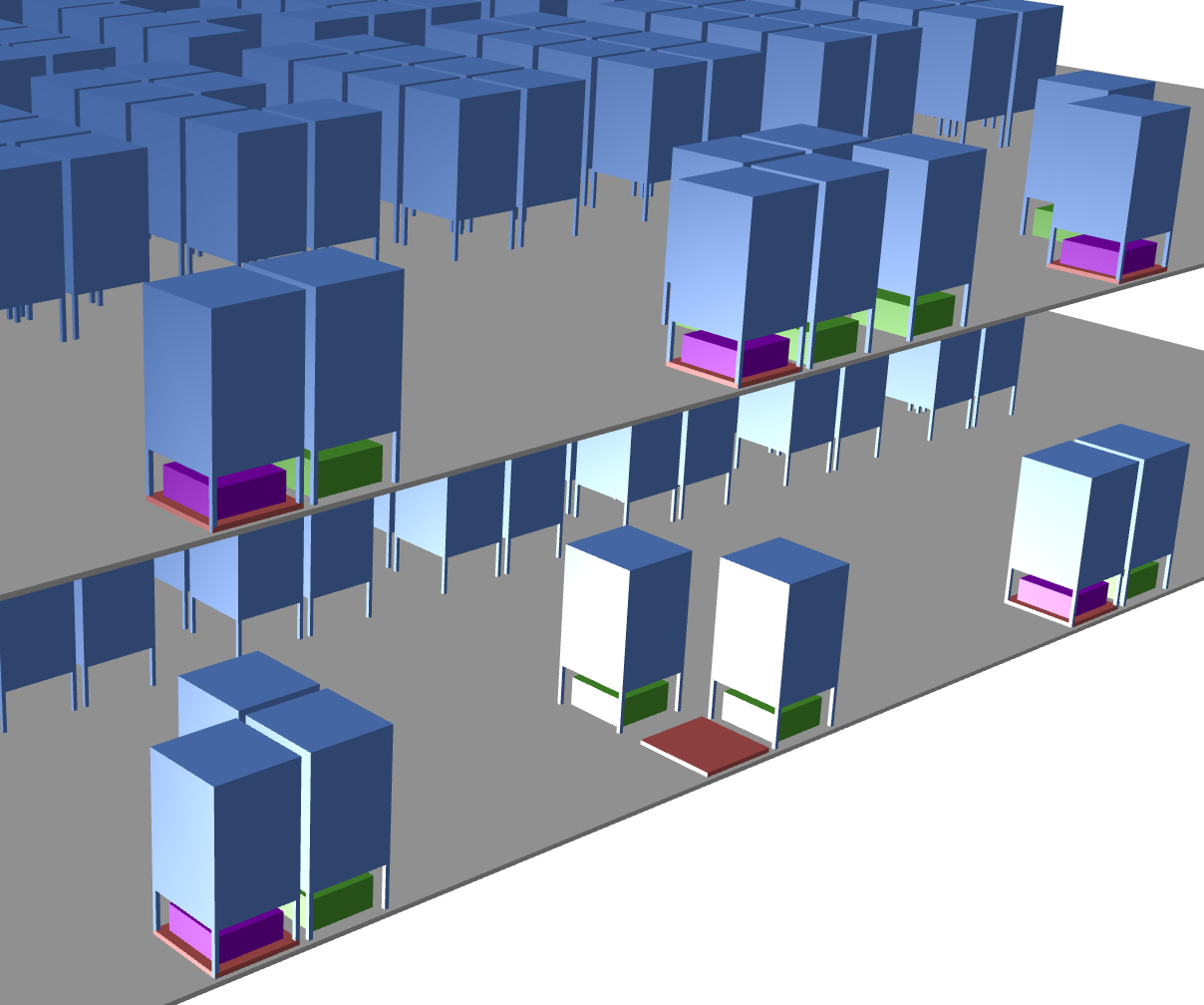}
		\caption{3D overview}
		\label{fig:pp_simulationscreenshot3D}
	\end{subfigure}
	\caption{Screenshots of the simulation visualization}
	\label{fig:pp_screenshots}
\end{figure}

Due to the focus on path planning a backlog of orders of constant length is available at all times, i.e., a new customer order is generated as soon as another one is finished. Hence, the system is being kept under pressure and robots may always have a task to execute. The same is done for generating new bundles of items such that the inventory does not deplete. The number of SKUs is also kept low (100 SKUs) to reduce the risk of stock-outs and to maintain a more stable system during the simulation horizon.

For each of the methods mentioned above a path planning engine is implemented as a wrapper that acts like an agent of the simulation. It is responsible for passing necessary information about the current state of simulation to the methods and coordinates the calls to the planning algorithms. In the update routine of the path planning agent (see Alg. \ref{alg:pp_pathplanneragent}) first the reservation table is reorganized to remove past reservations and then the embedded path planning algorithm is executed. This is only done, if there is no ongoing timeout and there is at least one robot requesting a new path. A robot will request a new path, if it is assigned to a new task with a new destination or the execution of its former path failed. Like mentioned before, the execution of a path may fail, if paths that are not collision-free are assigned to the robots. In this case the path planning engine will abort the execution of the path to avoid an imminent collision. For this reassurance again the reservation table is used.
\begin{algorithm}[h]
	\SetKwFunction{ExecPP}{ExecutePathPlanner}
	\caption{$\text{\textit{Update}}(\mIndexTime,\mSetBots,\mVarResTab,\mParPathPlanningTimeout)$}
	\label{alg:pp_pathplanneragent}
	\DontPrintSemicolon
	\LinesNumbered
	$\mReorganizeReservationTable{\mVarResTab}$\\
	\If{$\left( \mIndexTime' + \mParPathPlanningTimeout < \mIndexTime \right) \wedge \left( \exists \mIndexBot \in \mSetBots : \pi(\mIndexBot) = \emptyset \right)$}{
		\ExecPP{}\\
		$\mIndexTime' \leftarrow \mIndexTime$\\
	}
\end{algorithm}

\subsubsection{Layout} \label{subsubsec:pp_Layout}
The layout of the instances used for the experiment are built by a generator based on the work of \cite{Lamballais.2016}. Mainly three different areas can be identified within the layout (see Fig. \ref{fig:pp_instancelayout}). First, an inventory area is built, which contains all pod storage locations $\mSetBucketParkingSpots$ (indicated by blue squares). These are created in blocks of eight waypoints and connected by other waypoints by bidirectional edges. The waypoints of the aisles in between are connected by directional edges such that a cycle emerges around each block up to the complete inventory area. The directions of the edges are denoted by arrows. Hence, a robot, which is not carrying a pod, can almost freely navigate below stored pods, while a robot carrying a pod must adhere to a certain cyclic flow. This area is surrounded by a hall-area (long dashes) that serves as a highway for robots traveling from storage locations to stations. That offers some space to reduce congestion effects in front of the stations. The last area (short dashes) is used for buffering robots inbound for the stations. This area also contains all replenishment (yellow circles) and pick stations (red-circles). Each station has its own queue that is managed by a queue manager instead of path planning; i.e., as soon as a robot enters this area path planning gets deactivated and instead a queue manager is responsible for moving up robots towards the station while also exploiting shortcuts, if possible.

\begin{figure}[ht]
	\centering
	\includegraphics[width=0.95\textwidth]{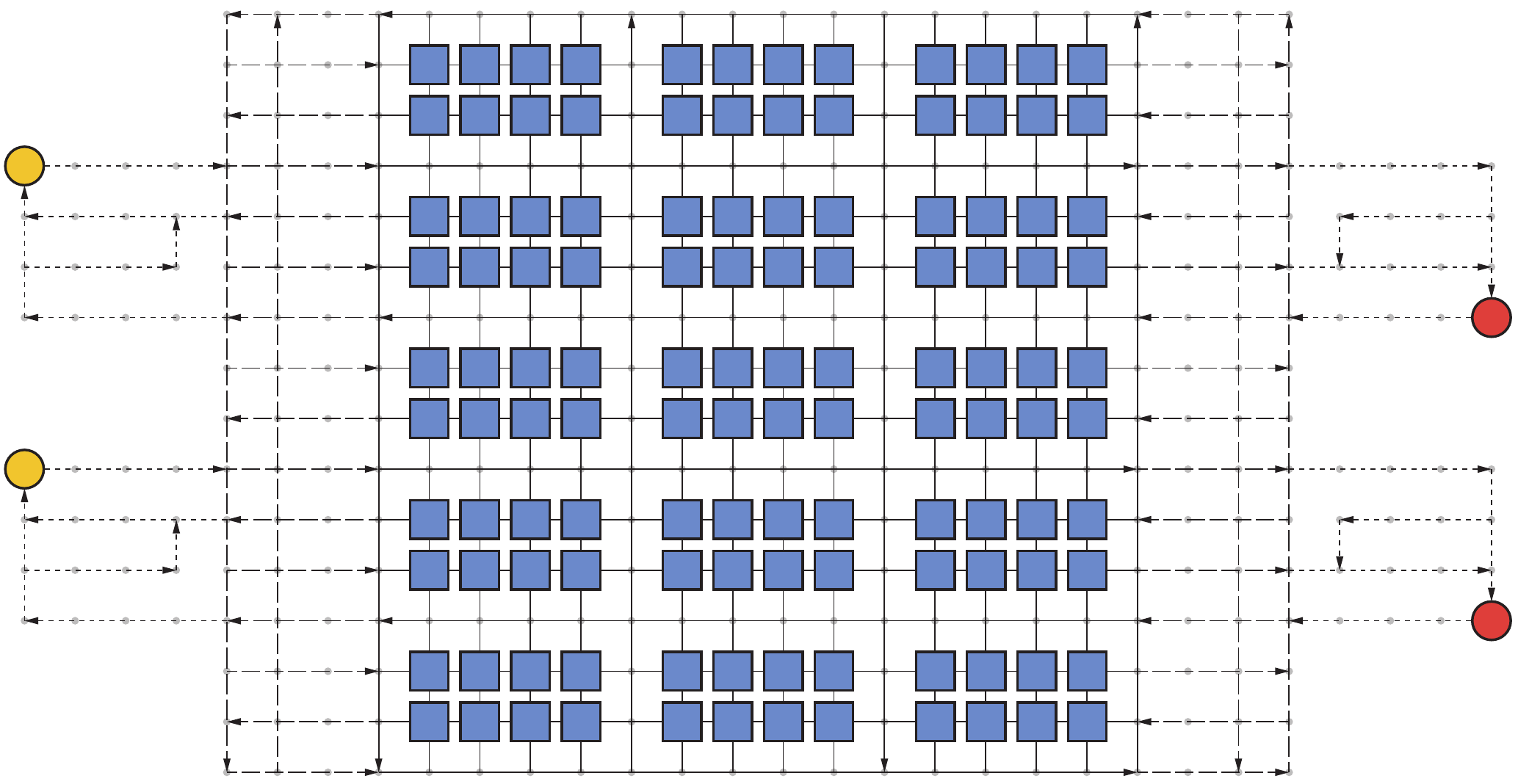}
	\caption{Basic layout of instances built by the generator}
	\label{fig:pp_instancelayout}
\end{figure}

\subsection{Experiment setup}

For the experiment we use 10 instances with different layouts (see Table \ref{tab:pp_instance-overview}). The names of the instances are derived from the number of tiers, replenishment stations, pick stations, robots and pods. All of them adhere to the layout described above, but with minor modifications. The instance 1-12-20-128-1965 uses a layout that surrounds the inventory area with the hall and buffer areas leading to more stations compared to the storage locations. Furthermore, instances 2-8-8-64-1100 and 3-12-12-96-1650 contain two, respectively, three floors connected by elevators. The elevators are positioned similar to the stations and transport one robot at a time from one floor to the next one in 10 seconds. Instance 1-6-14-106-1909 is shaped like a 'L' while 1-6-16-146-2726 contains 'holes' emulating obstacles of a warehouse building structure in its waypoint graph. The instances 1-4-16-144-1951 and 1-1-3-48-795 have all their stations positioned at only one side. The other instances use a form of the default layout described above in different sizes. For all instances roughly 85\% pods are used when compared to available storage locations. This allows for more options when dynamically determining a storage location each time a pod is brought back to the inventory. Additionally, we provide the ratio of robots per station. This is useful as a first intuition for the potential of congestion effects, i.e.: the more robots are used in less space the more conflicts may occur. Furthermore the number of waypoints in the graph is shown. Note that these are all waypoints of the system, including the storage locations shown separately as well as other special purpose waypoints (e.g. the ones used for stations and their queues).

\newcommand{\specialcell}[2][c]{%
  \begin{tabular}[#1]{@{}l@{}}#2\end{tabular}}
\begin{table}[ht]
	\setlength\tabcolsep{1mm}
	\caption{Characteristics of the instances used in the experiment}
	\begin{tabular}{lllllllll}\label{tab:pp_instance-overview}
		Name & tiers & bots & \specialcell{bots \\ per \\ station} & pods & \specialcell{storage \\ locations} & \specialcell{way- \\ points} & \specialcell{repl. \\ stations} & \specialcell{pick \\ stations} \\
		\hline
		1-1-3-48-795 & 1 & 48 & 12.0 & 795 & 936 & 2112 & 1 & 3 \\
		1-4-4-32-550 &  1 & 32 & 4.0 & 550 & 648 & 1640 & 4 & 4 \\
		1-4-16-144-1951 & 1 & 144 & 7.2 & 1951 & 2296 & 5528 & 4 & 16 \\
		1-6-14-106-1909 & 1 & 106 & 5.3 & 1909 & 2248 & 5654 & 6 & 14 \\
		1-6-16-146-2726 & 1 & 146 & 6.6 & 2726 & 3208 & 8120 & 6 & 16 \\
		1-8-8-64-1040 & 1 & 64 & 4.0 & 1040 & 1224 & 3064 & 8 & 8 \\
		1-8-8-96-1502 & 1 & 96 & 6.0 & 1502 & 1768 & 4104 & 8 & 8 \\
		1-12-20-128-1965 & 1 & 128 & 4.0 & 1965 & 2312 & 6088 & 12 & 20 \\
		2-8-8-64-1100 & 2 & 64 & 4.0 & 1100 & 1296 & 4144 & 8 & 8 \\
		3-12-12-96-1650 & 3 & 96 & 4.0 & 1650 & 1944 & 6216 & 12 & 12 \\
	\end{tabular}
\end{table}

All robots of the experiment share the same acceleration and deceleration rate of $0.5 \frac{m}{s^2}$ and a top-speed of $1.5 \frac{m}{s}$. The time needed for a full rotation is set to $2.5 s$. Robots and pods are emulated as moving circles with a diameter of $70 cm$, respectively $90 cm$. The times for picking up a pod, setting down a pod, storing a bundle of items ($\mParTimeIStationHandleUnitIndexed$) and picking a single item ($\mParTimeOStationHandleUnitIndexed$) are all constant and set to $3 s$, $3 s$, $10 s$ and $10 s$, respectively. The following default parameters are used for the different methods, which are chosen according to the results of a preceding grid search on a small subset of possible parameter values. For all methods the length of a wait step is set to $2 s$ while the timeout for a single path planning execution is set to $1 s$, i.e., the path planning algorithm is called at most once per second. For $\text{WHCA}^*_v$ and $\text{WHCA}^*_n$ the time window is set to $20 s$ and  $30 s$, respectively. For $\text{BCP}$ the biased cost is set to $1$. The search method used for $\text{CBS}$ is best first. The maximal node count for $\text{OD\&ID}$ is set to $100$.

For the assessment of performance we first look at the sum of item bundles stored and units picked at the replenishment and pick stations (handled units). This metric also resembles the work that is done by the system during simulation horizon, which relates to the throughput being a typical goal for such a parts-to-picker system. Hence, the implied goal for path planning is to generate paths that can be executed very fast such that the stations wait for robots bringing pods as little as possible. For more detailed insights we added the average length of the computed paths and the average time it took the robots to complete them (trip length \& trip time). At last we look at the wall-clock time consumed by the different methods (wall time).

\subsection{Simulation results}

In the following we discuss the computational results of the experiment described above. Each combination of method and instance is simulated for 24 hours with 10 repetitions to lessen the effect of randomness caused by other controllers and simulation components. In table \ref{tab:pp_results-main}, the arithmetic mean of the proposed metrics is given per method and across all instances. Additionally, the timeout of $1 s$ per path planning execution implies that the overall wall-clock time usable by a method is limited by the simulation horizon of $86.400 s$. In table \ref{tab:pp_results-handledunits-detailed} the handled units of the methods in average per instance are shown. This is done to allow further insights about the performance of the method related to the instance characteristics.

\begin{table}[h]
	\caption{Performance results for the different methods and metrics (averages across repetitions and instances, ordered by handled units)}
	\begin{tabular}{l|rrrr}\label{tab:pp_results-main}
		Method & Handled units & Trip length (m) & Trip time (s) & Wall time (s) \\
		\hline
		$\text{WHCA}^*_v$ & 81011.43 & 52.76 & 67.61 & 9555.28 \\
		$\text{BCP}$ & 80469.76 & 53.69 & 69.53 & 81027.86 \\
		$\text{FAR}_r$ & 78260.82 & 54.17 & 71.07 & 1272.06 \\
		$\text{WHCA}^*_n$ & 77326.95 & 53.58 & 71.27 & 2184.29 \\
		$\text{OD\&ID}$ & 75380.12 & 52.48 & 72.94 & 14780.25 \\
		$\text{FAR}_e$ & 71942.03 & 54.91 & 77.51 & 473.72 \\
		$\text{CBS}$ & 60205.38 & 53.31 & 90.21 & 65560.65 \\
	\end{tabular}
\end{table}

% Code to generate this table:
% LaTeXTableGenerator.exe input.csv \t de 0 mediumgreen/mediumyellow/mediumred/80,c1,c2,c3,c4,c5,c6,c7,c8,9,r1 mediumgreen/mediumyellow/mediumred/80,c1,c2,c3,c4,c5,c6,c7,c8,9,r2 mediumgreen/mediumyellow/mediumred/80,c1,c2,c3,c4,c5,c6,c7,c8,9,r3 mediumgreen/mediumyellow/mediumred/80,c1,c2,c3,c4,c5,c6,c7,c8,9,r4 mediumgreen/mediumyellow/mediumred/80,c1,c2,c3,c4,c5,c6,c7,c8,9,r5 mediumgreen/mediumyellow/mediumred/80,c1,c2,c3,c4,c5,c6,c7,c8,9,r6 mediumgreen/mediumyellow/mediumred/80,c1,c2,c3,c4,c5,c6,c7,c8,9,r7 mediumgreen/mediumyellow/mediumred/80,c1,c2,c3,c4,c5,c6,c7,c8,9,r8 mediumgreen/mediumyellow/mediumred/80,c1,c2,c3,c4,c5,c6,c7,c8,9,r9 mediumgreen/mediumyellow/mediumred/80,c1,c2,c3,c4,c5,c6,c7,c8,9,r10
\begin{table}[h]
	\setlength\tabcolsep{0.9mm}
	\caption{Handled units per instance and method (averages across repetitions, green $\equiv$ best / red $\equiv$ worst per row)}
	\begin{tabular}{l|rrrrrrr}\label{tab:pp_results-handledunits-detailed}
		Instance & $\text{WHCA}^*_v$ & $\text{BCP}$ & $\text{FAR}_r$ & $\text{WHCA}^*_n$ & $\text{OD\&ID}$ & $\text{FAR}_e$ & $\text{CBS}$ \\
		\hline
		1-1-3-48-795 & \cellcolor{mediumgreen!18!mediumyellow!80!white}30332 & \cellcolor{mediumyellow!99!mediumred!80!white}30186 & \cellcolor{mediumgreen!100!mediumyellow!80!white}30955 & \cellcolor{mediumgreen!21!mediumyellow!80!white}30359 & \cellcolor{mediumyellow!27!mediumred!80!white}29641 & \cellcolor{mediumgreen!75!mediumyellow!80!white}30763 & \cellcolor{mediumyellow!0!mediumred!80!white}29437 \\
		1-4-4-32-550 & \cellcolor{mediumgreen!100!mediumyellow!80!white}29820 & \cellcolor{mediumgreen!12!mediumyellow!80!white}28787 & \cellcolor{mediumyellow!64!mediumred!80!white}28230 & \cellcolor{mediumgreen!16!mediumyellow!80!white}28835 & \cellcolor{mediumyellow!75!mediumred!80!white}28357 & \cellcolor{mediumyellow!0!mediumred!80!white}27481 & \cellcolor{mediumgreen!87!mediumyellow!80!white}29669 \\
		1-4-16-144-1951 & \cellcolor{mediumgreen!100!mediumyellow!80!white}142482 & \cellcolor{mediumgreen!94!mediumyellow!80!white}139722 & \cellcolor{mediumgreen!88!mediumyellow!80!white}136946 & \cellcolor{mediumgreen!60!mediumyellow!80!white}124404 & \cellcolor{mediumgreen!46!mediumyellow!80!white}118436 & \cellcolor{mediumyellow!93!mediumred!80!white}94683 & \cellcolor{mediumyellow!0!mediumred!80!white}53156 \\
		1-6-14-106-1909 & \cellcolor{mediumgreen!100!mediumyellow!80!white}126878 & \cellcolor{mediumgreen!88!mediumyellow!80!white}124571 & \cellcolor{mediumgreen!88!mediumyellow!80!white}124391 & \cellcolor{mediumgreen!88!mediumyellow!80!white}124412 & \cellcolor{mediumgreen!77!mediumyellow!80!white}122263 & \cellcolor{mediumgreen!74!mediumyellow!80!white}121788 & \cellcolor{mediumyellow!0!mediumred!80!white}86965 \\
		1-6-16-146-2726 & \cellcolor{mediumgreen!100!mediumyellow!80!white}149142 & \cellcolor{mediumgreen!83!mediumyellow!80!white}146119 & \cellcolor{mediumgreen!93!mediumyellow!80!white}147903 & \cellcolor{mediumgreen!89!mediumyellow!80!white}147188 & \cellcolor{mediumgreen!75!mediumyellow!80!white}144540 & \cellcolor{mediumgreen!82!mediumyellow!80!white}145789 & \cellcolor{mediumyellow!0!mediumred!80!white}112812 \\
		1-8-8-64-1040 & \cellcolor{mediumgreen!100!mediumyellow!80!white}50313 & \cellcolor{mediumgreen!59!mediumyellow!80!white}49174 & \cellcolor{mediumyellow!97!mediumred!80!white}47456 & \cellcolor{mediumgreen!9!mediumyellow!80!white}47767 & \cellcolor{mediumyellow!92!mediumred!80!white}47299 & \cellcolor{mediumyellow!0!mediumred!80!white}44742 & \cellcolor{mediumgreen!29!mediumyellow!80!white}48334 \\
		1-8-8-96-1502 & \cellcolor{mediumgreen!100!mediumyellow!80!white}62120 & \cellcolor{mediumgreen!98!mediumyellow!80!white}61939 & \cellcolor{mediumgreen!45!mediumyellow!80!white}57862 & \cellcolor{mediumgreen!56!mediumyellow!80!white}58693 & \cellcolor{mediumgreen!32!mediumyellow!80!white}56857 & \cellcolor{mediumyellow!91!mediumred!80!white}53700 & \cellcolor{mediumyellow!0!mediumred!80!white}46633 \\
		1-12-20-128-1965 & \cellcolor{mediumgreen!100!mediumyellow!80!white}85463 & \cellcolor{mediumgreen!72!mediumyellow!80!white}82024 & \cellcolor{mediumgreen!60!mediumyellow!80!white}80497 & \cellcolor{mediumgreen!63!mediumyellow!80!white}80870 & \cellcolor{mediumgreen!37!mediumyellow!80!white}77644 & \cellcolor{mediumgreen!14!mediumyellow!80!white}74796 & \cellcolor{mediumyellow!0!mediumred!80!white}60705 \\
		2-8-8-64-1100 & \cellcolor{mediumgreen!100!mediumyellow!80!white}54383 & \cellcolor{mediumgreen!25!mediumyellow!80!white}53026 & \cellcolor{mediumyellow!66!mediumred!80!white}51944 & \cellcolor{mediumgreen!19!mediumyellow!80!white}52907 & \cellcolor{mediumyellow!70!mediumred!80!white}52021 & \cellcolor{mediumyellow!0!mediumred!80!white}50756 & \cellcolor{mediumgreen!89!mediumyellow!80!white}54188 \\
		3-12-12-96-1650 & \cellcolor{mediumgreen!100!mediumyellow!80!white}79183 & \cellcolor{mediumgreen!89!mediumyellow!80!white}78940 & \cellcolor{mediumyellow!71!mediumred!80!white}76425 & \cellcolor{mediumgreen!37!mediumyellow!80!white}77836 & \cellcolor{mediumyellow!85!mediumred!80!white}76744 & \cellcolor{mediumyellow!0!mediumred!80!white}74923 & \cellcolor{mediumgreen!90!mediumyellow!80!white}78968 \\
	\end{tabular}
\end{table}

In terms of handled units, $\text{WHCA}^*_v$ is the most successful one, followed by $\text{BCP}$. We can also see that the trip time highly relates to the number of handled units, i.e. the shorter the time is for completing a trip the more units are handled overall. This does not hold for the trip length, i.e. the length of the trips lies in a close range, but a shorter one is not necessarily faster. This is mainly impacted by the more important coordination of the robots. Thus, shorter trips may cause more congestion and longer waiting times for the robot while it is executing a path. The wall time consumed by the methods differs significantly. While $\text{FAR}_e$ in average uses less than ten minutes to plan the paths for all robots for 24 hours, going from 85.57s for the 1-4-4-32-550 layout up to 1008.59s for the 1-6-16-146-2726 layout. In contrast, BCP almost always uses the complete allowed runtime. This means that BCP does plan reasonably efficient paths, but is not able to completely resolve all conflicts up until the destination of all of them. This can also be seen when looking at the percentage of executions that ended in a timeout (see Tab. \ref{tab:pp_results-further-insights}). It is almost impossible for BCP to generate completely conflict-free paths for the layout instances of quite realistic size before the timeout of $1 s$. We can only observe this for the smallest instances of the set. However, the paths successfully planned by BCP until the timeout takes effect are competitive. The FAR methods cause the longest trip lengths, but the $\text{FAR}_r$ variant can still compete with the others in terms of handled units and trip time. This is especially interesting when looking at the wall time consumed by it. Hence, the FAR method is a candidate to consider when controlling instances much larger than the ones considered in this work. Furthermore, the strategy of FAR to avoid head to head collisions is working well (in comparison) for instances that are more crowded with robots, i.e. have a higher robot to station ratio (see Tables \ref{tab:pp_results-handledunits-detailed}). The rather poor performance of CBS is a reason of the method not being able to generate efficient paths within the time limit for the quite large instances of the pool. This even causes cascading congestion effects, if a robot is not assigned any path and will block others even longer. In contrast, we see a good performance of CBS for the smaller instances with a lower robot to station ratio (see Tab. \ref{tab:pp_results-handledunits-detailed}). The $\text{WHCA}^*_n$ variant is still performing well while only consuming roughly 23 \% wall-clock time of $\text{WHCA}^*_v$. However, a longer trip time is the result of the robots being forced to plan their trips based on the existing ones without the possibility to find overall improved paths. The method $\text{OD\&ID}$ achieves reasonable performance while consuming acceptable wall-clock time across all instances.

%\todo{integrate the following about low pile-on and so on?}
%Although the experiment is specifically setup to increase pressure on path planning methods by only using random mechanisms for all other problem components, it can be observed that path planning has a high impact on the system's performance. Especially, selecting pods at random leads to a low pile-on (units picked per pod). Thus, handling times at stations are short in average and cause the robots to be moving almost all the time.

\begin{table}[h]
	\caption{Extended method comparison}
	\centering
	\begin{tabular}{l|rrrr}\label{tab:pp_results-further-insights}
		Method & Station idle time & Timeouts & \multicolumn{2}{r}{Memory used (MB)} \\
		 & & & average & maximum \\
		\hline
		$\text{WHCA}^*_v$ & 44.8 \% & 7.0 \% & 120.63 & 227.62 \\
		$\text{BCP}$ & 46.2 \% & 97.1 \% & 86.01 & 147.76 \\
		$\text{FAR}_r$ & 46.4 \% & 0.0 \% & 105.43 & 193.88 \\
		$\text{WHCA}^*_n$ & 47.0 \% & 0.0 \% & 128.98 & 254.78 \\
		$\text{OD\&ID}$ & 48.2 \% & 0.2 \% & 99.22 & 183.13 \\
		$\text{FAR}_e$ & 50.1 \% & 0.0 \% & 101.26 & 193.11 \\
		$\text{CBS}$ & 56.5 \% & 65.9 \% & 77.00 & 125.60 \\
	\end{tabular}
\end{table}

Furthermore, we can observe that the idle time of the station, i.e. the time the station is not busy picking items, respectively not busy storing bundles, is another metric that is closely related to the trip time of the methods (see Tab. \ref{tab:pp_results-further-insights}). To some extent, this means that stations will idle less, if robots reach them faster. Regarding the fairly high idle times note that the experiment is designed in a way that increases pressure on the path planning components, i.e. the controllers for the other components are causing longer trips and reasonable times for handling pods at the stations. This is done, because there is a natural upper bound for handling items and item bundles at the stations given by the constant time it needs to process one unit of each. Hence, it is not possible to process more units than given by the following simple upper bounds, which would lead to a bottleneck obscuring the impact of the path planning components when reached. For pick stations the time for picking one item limits the throughput per hour ($UB^O_\mIndexStation := \frac{3600}{\mParTimeOStationHandleUnit{\mIndexStation}}$) while for replenishment stations it is limited by the time for storing one item bundle ($UB^I_\mIndexStation := \frac{3600}{\mParTimeIStationHandleUnit{\mIndexStation}}$). Summing up upper bounds of the stations leads to an overall upper bound for handled units for the system ($UB := \sum_{\mIndexStation \in \mSetOStations} UB^O_\mIndexStation + \sum_{\mIndexStation \in \mSetIStations} UB^I_\mIndexStation$).
Like mentioned before some of the methods reach the given runtime timeout much more than others. Looking at the average timeouts of the different methods across all instance we can observe that BCP almost always uses its complete runtime given. However, it still obtains results of reasonable quality. In contrast, CBS also reaches it's runtime very often but is not able to obtain efficient paths within time except for the small instances of the set. For the latter also less timeouts occur for CBS. Except for $\text{WHCA}^*_v$ all other methods virtually never run into a timeout. Interestingly $\text{OD\&ID}$ consumes more wall time than $\text{WHCA}^*_v$ but faces less timeouts. Since the timeouts per instance are very similar for $\text{WHCA}^*_v$ this suggests that certain states during the simulation horizon cause spikes in the wall time consumed that lead to these timeouts.
At last, we show the maximal memory consumed by the different method. This is the maximum across all instances. It is obtained by executing a reference simulation for each instance applying a random walk method and subtracting the resulting memory consumption from the memory consumed by the respective method for the same instance. Note that the measurement of the memory underlies inaccuracy caused by technical influences like the garbage collection. However, overall we can see that the memory consumed only differs between the methods in reasonable absolute numbers. Especially when looking at the maximal memory consumption across all instances we can observe that memory is not the limiting factor for the proposed methods considering todays typical hardware.

\begin{figure}[h]
	\centering
	\begin{subfigure}[b]{0.49\textwidth}
		\includegraphics[width=\textwidth]{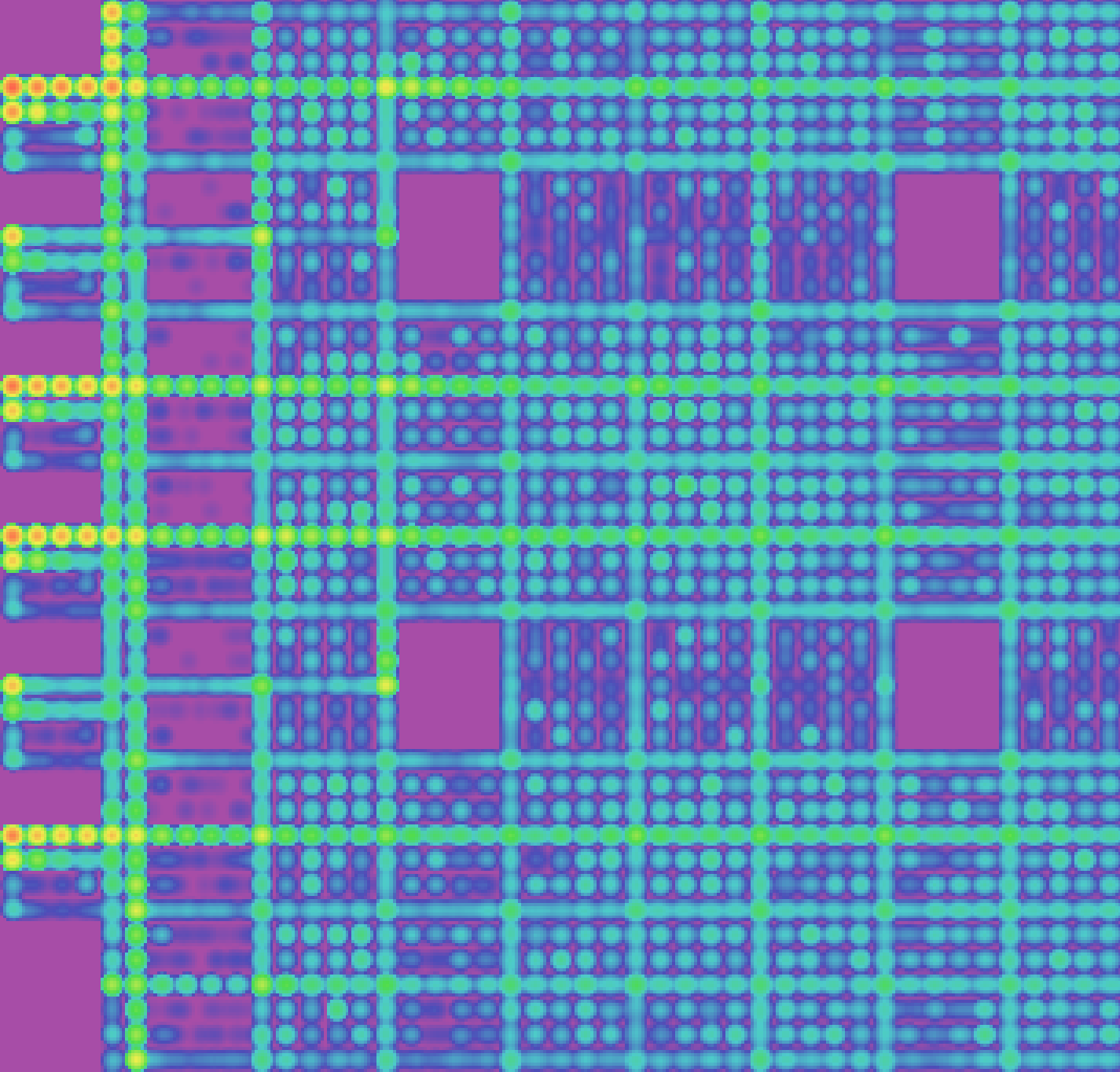}
		\caption{$\text{CBS}$}
	\end{subfigure}
	\begin{subfigure}[b]{0.49\textwidth}
		\includegraphics[width=\textwidth]{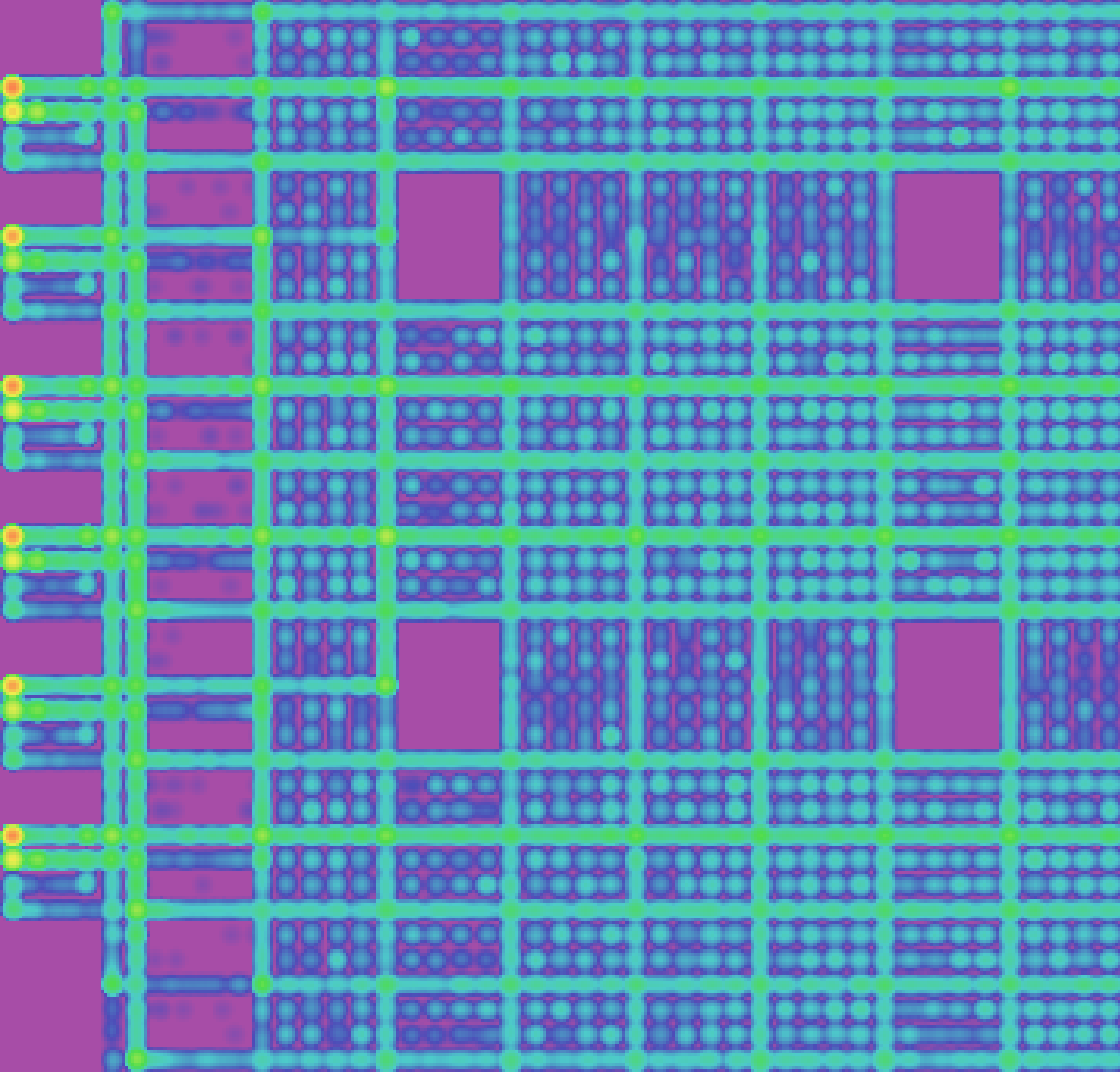}
		\caption{$\text{WHCA}^*_v$}
	\end{subfigure}
	\caption{Heat map of the robot positions over time for a part of instance 1-6-16-146-2726 (purple $\equiv$ low, red $\equiv$ high)}
	\label{fig:pp_heatmaps}
\end{figure}

The impact of path planning on the systems overall performance can also be seen in more detail when comparing the movement of the robots between $\text{WHCA}^*_v$ and $\text{CBS}$ for a large instance by using the heatmaps shown in Fig. \ref{fig:pp_heatmaps}. The heatmaps show the positions of all robots that are periodically polled throughput the simulation horizon. A logarithmic scale from purple as the coldest color to red as the hottest is used to render the values. Hence, warmer colors indicate areas where robots have been observed more frequently. In the case of $\text{CBS}$ the highest conflicting area in terms of congestion are the exits of the stations (stations are positioned all around the inventory area for this instance). When applying $\text{WHCA}^*_v$ to the same instance robots are able to leave the stations quickly and spent more time within the inventory area storing and retrieving pods. This is specifically indicated by less ``hot dots'', which can especially be observed at intersection waypoints. Overall the movement behavior of $\text{WHCA}^*_v$ is more 'fluent' along the main axes from the inventory area towards the stations and back. We can observe similar effects for the other methods and instances. Thus, ensuring a fluent robot movement is an objective of efficient path planning.

	%%%%%%%%%%%%%%%%%%%%%%%%%%%%%%%%%%%%%%%%%%%%%%%%%%%%%%%%%
	%%%%%%%%%%%%%%%%%%%%%%%%%%%%%%%%%%%%%%%%%%%%%%%%%%%%%%%%%
	%%%%%%%%%%%%%%%%%%%%%%%%%%%%%%%%%%%%%%%%%%%%%%%%%%%%%%%%%
	
	\section{Conclusion}
	In this work we proposed a generalized problem definition of the MAPF problem that allows the application of methods for the field of Robotic Mobile Fulfillment Systems. Additionally, we have proven that $A^*$ is complete and admissible in the search space of our problem. Further on, we describe necessary modifications to the $A^*$-based algorithms previously applied to MAPF. Based on the computational results from our simulation framework, $\text{WHCA}^*_v$ performed the best according to the proposed metrics, but does not scale as well as the $\text{FAR}$ methods in terms of wall-clock time. Similarly, the time consumption of $\text{CBS}$ negatively impacted the performance for the large instances in our set. In contrast, $\text{FAR}$ is applicable even for very large instances.

In future we want to investigate the combination of the proposed methods with control mechanisms for the other decision components, such as \textit{task allocation}. At this, we expect that methods have mutual dependencies. Hence, effective methods have to be evaluated that cooperate best to achieve a globally efficient system. Furthermore, we observe the impact of the layout characteristics and size on the performance, which suggests a more detailed investigation of the dependencies of these.
%At last, the characteristics of executing the methods repeatedly (up to once per second in our setup) increase our interest in applying an on-line algorithm configuration approach, as proposed by \cite{Fitzgerald.2014}, to better adapt the parameters of the methods to the instances and the current situation during execution. 

	%%%%%%%%%%%%%%%%%%%%%%%%%%%%%%%%%%%%%%%%%%%%%%%%%%%%%%%%%
	%%%%%%%%%%%%%%%%%%%%%%%%%%%%%%%%%%%%%%%%%%%%%%%%%%%%%%%%%
	%%%%%%%%%%%%%%%%%%%%%%%%%%%%%%%%%%%%%%%%%%%%%%%%%%%%%%%%%
	
	\section*{Acknowledgements}
	We would like to thank Tim Lamballais for providing us with the layout concept used for the instances of the experiments. Additionally, we thank the Paderborn Center for Parallel Computing (PC$^2$) for the use of their HPC systems for conducting the experiments. Marius Merschformann is funded by the International Graduate School - Dynamic Intelligent Systems, Paderborn University.
	
	%%%%%%%%%%%%%%%%%%%%%%%%%%%%%%%%%%%%%%%%%%%%%%%%%%%%%%%%%
	%%%%%%%%%%%%%%%%%%%%%%%%%%%%%%%%%%%%%%%%%%%%%%%%%%%%%%%%%
	%%%%%%%%%%%%%%%%%%%%%%%%%%%%%%%%%%%%%%%%%%%%%%%%%%%%%%%%%
	
	%\newpage
	\bibliographystyle{plain}      % basic style, author-year citations
	\bibliography{public/bibliography}   % name your BibTeX data base
	\newpage
	
	\appendix
	\section{Theorem} \label{sec:pp_Theorem}

\begin{theorem}
Let $c(\mIndexState_1,\mIndexState_2)$ be the cost of the arc $(\mIndexState_1,\mIndexState_2)$ in search space $S$ of MAPFWR. Then it holds a finite and positive lower bound $\delta$, $c(\mIndexState_1,\mIndexState_2) \ge \delta > 0$.
\end{theorem}
\begin{proof}
Let $r$ be $argmax_{\mIndexBot \in \mSetBots} \mParMaxVelocity{r}$, $ T^R$ be $\dfrac{\max_{i\in \mSetBots \cup \mSetBuckets} \max_{i'\in \mSetBots \cup \mSetBuckets \setminus \left\{ i \right\}} \mParRadius{i} + \mParRadius{i'}}{\mParMaxVelocity{\mIndexBot}}$ and $\delta$ be $min(\mParTimeWaiting,T^R)$. So $\delta > 0$ since $\mParTimeWaiting \in \mContinuousPositive$, $\forall \mIndexBot \in \mSetBots: \mParRadius{\mIndexBot},\mParMaxVelocity{\mIndexBot} \in \mContinuousPositive$ and $\forall \mIndexBucket \in \mSetBuckets: \mParRadius{\mIndexBucket} \in \mContinuousPositive$. Now we look at the possible actions occurring on the arc $(\mIndexState_1,\mIndexState_2)$.
\begin{enumerate}
	\item If there is a wait action then $c(\mIndexState_1,\mIndexState_2)=\mParTimeWaiting \ge \delta >0$
	\item If there is a move action through the arc $(\mIndexWaypoint_1,\mIndexWaypoint_2)$ and a rotation occurs on $\mIndexWaypoint_1$ with angle $\varphi$ then\\
	$c(\mIndexState_1,\mIndexState_2)=\mGetTimeRotation{\mIndexBot}{\varphi}{\mParMaxAngularVelocity{\mIndexBot}}+\mGetTimeDrive{\mIndexBot}{\mGetDistanceTimeIndependent{\mIndexWaypoint_1}{\mIndexWaypoint_2}}{\mParAcceleration{\mIndexBot}}{\mParMaxVelocity{\mIndexBot}}{\mParDeceleration{\mIndexBot}}$\\
	$\ge \dfrac{\varphi}{\mParMaxAngularVelocity{\mIndexBot}}+\dfrac{\mGetDistanceTimeIndependent{\mIndexWaypoint_1}{\mIndexWaypoint_2}}{\mParMaxVelocity{\mIndexBot}} \ge \dfrac{\mGetDistanceTimeIndependent{\mIndexWaypoint_1}{\mIndexWaypoint_2}}{\mParMaxVelocity{\mIndexBot}} \ge T^R \ge \delta >0$
	\item If there is a move action through the arc $(\mIndexWaypoint_1,\mIndexWaypoint_2)$ and no rotation occurs on $\mIndexWaypoint_1$, moreover, $\mIndexWaypoint_0$ is the last stopping node where a rotation occurs with angle $\varphi$ then\\
	$c(\mIndexState_1,\mIndexState_2)=\mGetTimeRotation{\mIndexBot}{\varphi}{\mParMaxAngularVelocity{\mIndexBot}}+\mGetTimeDrive{\mIndexBot}{\mGetDistanceTimeIndependent{\mIndexWaypoint_0}{\mIndexWaypoint_2}}{\mParAcceleration{\mIndexBot}}{\mParMaxVelocity{\mIndexBot}}{\mParDeceleration{\mIndexBot}}\\
	-\mGetTimeDrive{\mIndexBot}{\mGetDistanceTimeIndependent{\mIndexWaypoint_0}{\mIndexWaypoint_1}}{\mParAcceleration{\mIndexBot}}{\mParMaxVelocity{\mIndexBot}}{\mParDeceleration{\mIndexBot}}$\\
	$\ge \dfrac{\varphi}{\mParMaxAngularVelocity{\mIndexBot}}+\dfrac{\mGetDistanceTimeIndependent{\mIndexWaypoint_1}{\mIndexWaypoint_2}}{\mParMaxVelocity{\mIndexBot}} \ge \dfrac{\mGetDistanceTimeIndependent{\mIndexWaypoint_1}{\mIndexWaypoint_2}}{\mParMaxVelocity{\mIndexBot}} \ge T^R \ge \delta >0$
\end{enumerate}
Thus, $c(\mIndexState_1,\mIndexState_2) \ge \delta >0$.
\end{proof}
			
In the proof, $\delta$ is selected as the smallest cost of the changed state for each robot, which is either the waiting time $\mParTimeWaiting$ or the time $T^R$ for the robot with highest speed to go through the shortest arc. According to Eq. \eqref{eq:pp_edgeLength}, $\max_{i\in \mSetBots \cup \mSetBuckets} \max_{i'\in \mSetBots \cup \mSetBuckets \setminus \left\{ i \right\}} \mParRadius{i}+\mParRadius{i'}$ is the lower bound of the length of $(\mIndexWaypoint1,\mIndexWaypoint2)$. In the first case, there is no action, which is shorter than $\delta$, since $\delta$ is $min(\mParTimeWaiting,T^R)$. In the second case, the time for going through an arc is longer than $T^R$, and likewise for the third case. Note that in the third case, the time for the path from $\mIndexWaypoint1$ to $\mIndexWaypoint2$ is calculated through the path from $\mIndexWaypoint0$ to $\mIndexWaypoint2$ minus the time for the path from $\mIndexWaypoint_0$ to $\mIndexWaypoint_1$. Therefore, $\delta$ is the finite and positive lower bound for all actions in search space $\mSetSearchSpace$.

	%\newpage
	%\input{todos.tex}
	
	%%%%%%%%%%%%%%%%%%%%%%%%%%%%%%%%%%%%%%%%%%%%%%%%%%%%%%%%%
	%%%%%%%%%%%%%%%%%%%%%%%%%%%%%%%%%%%%%%%%%%%%%%%%%%%%%%%%%
	%%%%%%%%%%%%%%%%%%%%%%%%%%%%%%%%%%%%%%%%%%%%%%%%%%%%%%%%%
	
\end{document}